\documentclass[10pt]{article} % For LaTeX2e
\usepackage[accepted]{tmlr}
\usepackage{amsmath, amsthm, amsfonts,bm}

\def\eqref#1{equation~\ref{#1}}
\def\1{\bm{1}}

\DeclareMathAlphabet{\mathsfit}{\encodingdefault}{\sfdefault}{m}{sl}
\SetMathAlphabet{\mathsfit}{bold}{\encodingdefault}{\sfdefault}{bx}{n}

\newcommand{\R}{\mathbb{R}}

\usepackage{changepage}
\newcommand{\red}[1]{\textcolor[rgb]{1.00,0.00,0.00}{#1}}
\newcommand{\blue}[1]{\textcolor{black}{#1}}
\newcommand{\ie}{\emph{i.e.}}
\newcommand{\eg}{\emph{e.g.}}
\newcommand{\Rf}{\mathfrak{R}}
\newcommand{\Fc}{\mathcal{F}}
\newcommand{\Tc}{\mathcal{T}}
\newcommand{\Lc}{\mathcal{L}}
\newcommand{\Xc}{\mathcal{X}}

\newcommand{\Wc}{\mathcal{W}}

\newcommand{\Eb}{\mathbb{E}}
\newcommand{\Rb}{\mathbb{R}}
\newcommand{\Sb}{\mathbb{S}}
\newcommand{\Gc}{\mathcal{G}}
\newcommand{\supp}{\textrm{supp}}
\newcommand{\dom}{\textrm{dom }}

\newcommand{\newl}{\nonumber \\}

\newcommand{\ppos}{p_+}
\newcommand{\pmargin}{p_d}
\newcommand{\squares}{p_\times}

\usepackage{amsfonts,amsmath, amssymb,amsthm}
\newtheorem{thm}{Theorem}%[section] % An environment for writing a theorem.
\newtheorem{definition}[thm]{Definition}

\newtheorem{assump}[thm]{Assumption}

\usepackage{wrapfig}
\usepackage[hidelinks]{hyperref}
\usepackage{url}
\usepackage{graphicx}
\usepackage[ruled,vlined,linesnumbered,norelsize]{algorithm2e}
\usepackage{booktabs}  
\usepackage{diagbox}
\usepackage{caption}
\usepackage{subcaption}
\usepackage{multirow}
\usepackage{enumitem}
\setlist{nolistsep}
\usepackage{soul}
\usepackage[switch]{lineno}

\newcommand{\vio}[1]{\textcolor{black}{#1}}
\newcommand{\teal}[1]{\textcolor{black}{#1}}
\usepackage{comment}
\usepackage{url}
\usepackage{amsthm,amsmath}
\usepackage{thmtools}
\usepackage{thm-restate}
\newtheorem{lem}[thm]{Lemma}
\usepackage{cleveref}
\usepackage{tcolorbox}

\title{$f$-MICL: Understanding and Generalizing InfoNCE-based Contrastive Learning}

\author{\name Yiwei Lu \thanks{Equal contribution} \email yiwei.lu@uwaterloo.ca \\
      \addr School of Computer Science\\
      University of Waterloo \\
      Vector Institute
      \AND
      \name Guojun Zhang \footnotemark[1] \email guojun.zhang@huawei.com \\
      \addr Huawei Noah's Ark Lab
      \AND
      \name Sun Sun \email sun.sun@nrc-cnrc.gc.ca \\
      \addr School of Computer Science\\
      University of Waterloo \\
      National Research Council Canada 
      \AND
      \name Hongyu Guo \email hongyu.guo@uottawa.ca\\
      \addr National Research Council Canada  \\
      University of Ottawa 
      \AND
      \name Yaoliang Yu \email yaoliang.yu@uwaterloo.ca\\
      \addr School of Computer Science \\
      University of Waterloo \\
      Vector Institute
}

\def\month{10}  % Insert correct month for camera-ready version
\def\year{2023} % Insert correct year for camera-ready version
\def\openreview{\url{https://openreview.net/forum?id=ZD03VUZmRx}} % Insert correct link to OpenReview for camera-ready version

\begin{document}

\maketitle

\begin{abstract}
In self-supervised contrastive learning, a widely-adopted  objective function is InfoNCE, which uses the heuristic cosine similarity for the representation comparison, and is closely related to maximizing the Kullback--Leibler (KL)-based mutual information. In this paper, we aim at answering two intriguing questions: (1) Can we go beyond the KL-based objective? (2) Besides the popular cosine similarity, can we design a better similarity function? We provide answers to both questions by generalizing  the KL-based mutual information to the $f$-\textbf{M}utual \textbf{I}nformation in  \textbf{C}ontrastive \textbf{L}earning ($f$-MICL) using the $f$-divergences. To answer the first question,
    we provide a wide range of $f$-MICL objectives which share the nice properties of InfoNCE (e.g., alignment and uniformity), and meanwhile result in similar or even superior performance. For the second question, assuming that the joint feature distribution is proportional to the Gaussian kernel, we derive an \textbf{$f$-Gaussian similarity} with better interpretability and empirical performance. Finally, we identify close relationships between the  $f$-MICL objective and several popular InfoNCE-based objectives. 
    Using benchmark tasks from both vision and natural language, we empirically 
    evaluate $f$-MICL with different $f$-divergences on various architectures (SimCLR, MoCo, and MoCo v3) and datasets. We observe that $f$-MICL generally outperforms the benchmarks and the best-performing $f$-divergence is task and dataset dependent.
\end{abstract}

\section{Introduction}
%\vspace{-0.3em}
Recent advances in self-supervised learning aim at learning similar representations from different augmented views of the same data sample. However, naively implementing this idea would easily make representations converge to some trivial constant  (i.e., feature collapse) in practice. To address this problem, researchers propose new algorithms either from the \emph{model architecture} perspective or the \emph{training objective} perspective.
% Two common solutions to  such problem are either from the \emph{architecture design} or the \emph{objective design} perspective. 
The former method (e.g., \cite{byol,siamsiam,zhang2021zero}) applies techniques such as  stop gradient or predictor module to create asymmetry in networks, while the latter method encourages the contrastiveness between similar (positive) and dissimilar (negative) sample pairs through the objective design.
%(i.e., pair of different images). 
%We provide a taxonomic illustration in
%Figure \ref{fig:cl} for the summary of the recent self-supervised learning approaches.

% focus on the \emph{objective design} of contrastive learning. 

\begin{comment}
\begin{figure*}
    \centering
    \includegraphics[width=0.9\textwidth]{images/cl2.png}
    \vspace{-0.5em}
    \caption{A taxonomy of self-supervised learning methods. We aim at studying two questions regarding the mutual information based contrastive loss: (1) Can we go beyond the KL-divergence based loss ? (2) Can we design other similarity functions?}
    \label{fig:cl}
    \vspace{-1.0em}
\end{figure*}
\end{comment}
%Contrastive learning has attracted a surge of attention recently due to its success in learning informative representation for image recognition, natural language understanding, and reinforcement learning \citep{logeswaran2018efficient,ChenKNH20,he2020momentum,srinivas2020curl}.
%Such unsupervised learning paradigm encourages the contrastiveness between similar and dissimilar sample pairs. Specifically, the feature embeddings of similar sample pairs are expected to be close while those of dissimilar sample pairs are expected to be far apart.

In this paper, we intend to deepen our understanding of contrastive learning by generalizing the current objective design.
%For the objective design,
%In terms of \emph{objective design}, 
To achieve self-supervised contrastive learning,
existing objectives are proposed from different perspectives such as
%works solve the problem 
the mutual information  (e.g., InfoNCE \citep{wu2018unsupervised,oord2018representation,ChenKNH20,henaff2019data,he2020momentum} ), the information redundancy  (e.g., Barlow Twins \citep{barlow}), and the regularization (e.g., VICReg \citep{vicreg}). 
%We focus on the former, i.e., the  InfoNCE objective, 
%a softmax cross-entropy loss, a.k.a.~InfoNCE, has been widely used  \citep{wu2018unsupervised,oord2018representation,ChenKNH20,henaff2019data,he2020momentum}. 
In particular, the InfoNCE objective is widely used,
which aims to maximize the probability of picking a similar sample pair 
%(with \textbf{cosine similarity}) 
among a batch of sample pairs, and can be interpreted as a lower bound of the mutual information (MI) between two views of samples \citep{oord2018representation,bachman2019learning,tian2020contrastive,TschannenDRGL20}. This is consistent with the well-known ``InfoMax principle" \citep{linsker1988self}. 
%GZ: data samples sound redudant 
To measure the similarity between sample pairs, cosine similarity is usually adopted.
%in InfoNCE.

To attain the aforementioned goals of better understanding and generalizing contrastive learning, we here focus on the widely-adopted InfoNCE objective, and aim at two questions regarding it: 
%framework which deserves to be answered: 

%\begin{adjustwidth}{1cm}{}
\textit{(1) MI is essentially the Kullback--Leibler (KL) divergence between the joint distribution and the product of the marginal distributions. Is this KL-based objective optimal? If not, can we go beyond the KL-based objective?}

\textit{(2) Besides the commonly used cosine similarity for measuring the distance between samples, can we provide a better similarity function with a theoretical basis?
}
%\end{adjustwidth}
%Is cosine similarity the optimal function to measure representation similarity? 

To answer the above two questions, we  generalize the KL-based mutual information to the broader $f$-divergence family \citep{ali1966general,csiszar1967information}, and propose the benchmark of $f$-mutual information in contrastive learning ($f$-MICL).
% To answer the first question, we explore the possibility of generalizing the KL-based mutual information to the broader $f$-divergence family \cite{ali1966general,csiszar1967information}, and propose the benchmark of $f$-mutual information in contrastive learning ($f$-MICL). 
By searching through a wide range of $f$-divergences, we observe that the KL divergence is not always the best, and several other $f$-divergences in fact show similar or even superior performance in practice. 
%This suggests that the generalization to $f$-MICL is indeed valuable.
%Moreover, we find that nice properties of InfoNCE loss (alignment and uniformity \cite{wang2020understanding}) can be naturally extended to $f$-MICL, both theoretically and empirically. 

For the second question, while it is challenging to provide an answer based on the InfoNCE objective, it is possible to derive a proper similarity function under the $f$-MICL framework.
% we can start from the definition of $f$-mutual information and derive the optimal similarity function from its variational lower bound. 
By assuming that the joint feature distribution is proportional to the popularly-adopted Gaussian kernel, we propose a novel \emph{$f$-Gaussian similarity} function that enjoys better empirical performance.

Finally, we show the generalization  of $f$-MICL by drawing connections between the $f$-MICL objective and several popular InfoNCE-based objectives (e.g., SimCLR\citep{ChenKNH20}, MoCo\citep{he2020momentum}, and Alignment and Uniformity (AU) \citep{wang2020understanding}). We identify that those objectives are  closely related  to $f$-MICL: Alignment and Uniformity (AU) \citep{wang2020understanding} can be treated as a special case, and SimCLR \citep{ChenKNH20} and MoCo \citep{he2020momentum} are upper bounds for a transformed $f$-MICL.  These results provide a different angle to better understand InfoNCE.
% Finally, we draw interesting connections between InfoNCE-based objectives with our $f$-MICL objectives. We show that InfoNCE objective can be regarded as a special case of the broader $f$-MICL family, which also provides key insights to understand InfoNCE loss intuitively. 
Moreover, we show  both theoretically and empirically that nice properties of InfoNCE  (e.g., alignment and uniformity \citep{wang2020understanding}) can be naturally extended to $f$-MICL.

%In contrastive learning, it is also crucial to design a proper similarity function between sample pairs, to improve the performance. Based on the variational lower bound of $f$-MI, we provide a \emph{principled} way to design the similarity function. This lays solid theoretical foundations for contrastive learning and also draws interesting connection with kernel methods \citep{powell87}. Comparably, most existing similarity functions are designed \emph{ad hoc}, such as cosine similarity \citep{ChenKNH20, he2020momentum}, bilinear functions \citep{oord2018representation, tian2020contrastive, henaff2019data}, and neural networks \citep{hjelm2018learning}.
% is crucial for the evaluation of the contrastiveness of similar and dissimilar sample pairs. 
%GZ: contrastiveness sounds too intimidating

%By assuming that the joint feature distribution of two similar samples is proportional to a Gaussian kernel, we derive an optimal similarity function for practical use, which resembles the well-known radial basis function (RBF) kernels \citep{powell87}.

%See more details in \S~\ref{sec:design}.
We summarize our main contributions as follows:
% We summarize our $f$-MICL framework in Figure \ref{fig:f-MICL-intro}. 
% We measure the Gaussian similarity between these pairs and pass them to $f'$ and $-f^* \circ f'$ separately. After averaging and maximization, the two terms imply alignment and uniformity, and their sum becomes $f$-mutual information (see Def.~\ref{def:fMI}). 

% We can naturally decompose our $f$-MICL function into two terms, which correspond to the properties of the \textit{alignment} and the \textit{uniformity}. Such characterization has been revealed in \citet[][Thm.~1]{wang2020understanding} for the InfoNCE loss. Compared with \citet{wang2020understanding}, our result applies to a wide range of the $f$-divergence family, and it does not rely on the limit of an infinite number of dissimilar samples. This allows us to explore the space of $f$-mutual information and improve the performance of InfoNCE-based contrastive learning.

\begin{itemize}[topsep=0pt,parsep=0pt,itemsep=0pt]
\item  
%We propose a novel framework for contrastive learning (called $f$-MICL) by encouraging the contrastiveness of positive and negative pairs with a general $f$-divergence family.

Motivated by InfoNCE, we propose a general framework for contrastive learning by extending the MI to the general $f$-MI, which provides a wide range of objective choices.
%and studying its variational lower bound. 
\item 
Instead of using heuristic similarity functions, we provide a novel similarity function, called \emph{$f$-Gaussian similarity}, based on the convex conjugate and an assumption on the joint feature distribution.
% Starting from this variational lower bound, we provide a principled way to choose the similarity function. Our method shows interesting connection between contrastive learning and kernel methods.
\item 
We identify close relationships between our $f$-MICL objective and several InfoNCE-based contrastive learning objectives.
% We provide interesting insights on understanding and generalizing the InfoNCE objective using our $f$-MICL benchmark.  
%\red{which is significant for practical use.}
%our $f$-MICL objective can be estimated with finite samples and we give the corresponding error bound. 
%\red{Additionally, we give the error bound when estimating the $f$-MICL objective with a finite number of samples.}
% GZ: don't be too detailed, don't be repetitive
\item Empirically,
we show that $f$-MICL achieves notable improvement over
benchmarks on various datasets, and the best-performing $f$-divergence depends on the specific task and dataset. In addition, our proposed $f$-Gaussian similarity  consistently outperforms the cosine similarity.%
% (1) we provide more choices of objectives beyond the InfoNCE loss. We observe
% that such choices are task and dataset dependent, and can be justified through a validation set in practice. (2) we provide a similarity function ($f$-Gaussian similarity) beyond cosine similarity, which improves the performance of $f$-MICL objectives in general.
% \item We provide a simple, effective and flexible contrastive learning framework based on $f$-mutual information, called \emph{$f$-mutual information contrastive learning} ($f$-MICL).
% \item We prove that $f$-MICL naturally gives uniformity and alignment.
% \item Based on our theory, we give guidance on how to sample negative pairs and how to choose the hyperparameters.
% \item Experimentally, we show that $f$-MICL achieves notably improvement over InfoNCE-based frameworks, such as SimCLR \citep{ChenKNH20} and MoCo \citep{he2020momentum}. The proposed similarity function outperforms the cosine similarity.  
\end{itemize}

% \textbf{Notations.} We assume that a dominating measure $\lambda$ (e.g.~Lebesgue) is given and all other probability measures are represented as some density w.r.t.~$\lambda$. 
% Given the joint density $p(x,y)$, we denote $p(x)  = \int p(x,y) \mathrm{d}\lambda(y)$ and $p(y) = \int p(x,y) \mathrm{d}\lambda(x)$ as the marginals. We use $\supp(\cdot)$ to denote the support of a distribution, and $f^*$ to denote the conjugate of function $f$. Every norm presented is Euclidean. We use $x^g := g(x)$ as the shorthand notation for the feature embedding, with $x$ a raw sample. The notation $p_{d}$ stands for the data distribution, and $\squares := p_d\otimes p_d$ means its self product. We denote $p_+$ as the distribution of \emph{positive pairs}, \ie, two samples with similar feature embeddings. The symbol $\circ$ denotes function composition.

%\vspace{-0.3em}
\section{$f$-Mutual Information}\label{sec:prem}
%\vspace{-0.3em}
%\blue{Directly introduce the notations in the feature space}

%\todoH{I am a bit confused by the term feature embedding or feature representation. I think contrastive learning is more for data/sample  representation, and the contrastiveness is on the data embedding  or representation space.}

%We provide preliminaries for our framework. 

To provide answers to the above two questions regarding the InfoNCE objective,
we first extend the KL-based mutual information to the more general $f$-mutual information.
%before applying it to contrastive learning. 
The definition of the $f$-mutual information ($f$-MI)
%The $f$-mutual information ($f$-MI) 
%between a pair of random variables $X$ and $Y$ 
is as follows:

%Contrastive learning is a popular \emph{self-supervised} method for learning data representations. In contrastive learning, we expect that similar sample pairs to be close to each other in the embedding space while uncorrelated pairs to be far away. Assume that this distribution is symmetric w.r.t.~the two random variables, then the resultant two marginals both follow the data distribution $p_d$ \citep{wang2020understanding}.

%For a pair of random variables $(x, y)$, 
%We propose the $f$-mutual information framework for contrastive learning. 

\begin{definition}[\textbf{$f$-mutual information, \citealt{csiszar1967information}}] \label{def:fMI} % , esposito2020robust
Consider a pair of random variables $(X, Y)$ with joint density function $p(x,y)$ and marginal densities $p(x)$ and $p(y)$. 
The $f$-mutual information $I_f$ between $X$ and $Y$ is defined as
\begin{align}
\begin{split}
&I_f(X;Y) := %D_{f}\left(p(x,y) \|p(x) p(y) \right) 
\int f\left(\frac{p(x, y)}{p(x)p(y)} \right) p(x)p(y) \cdot \mathrm{d}\lambda(x,y),
\end{split}
\end{align}
where $f : \Rb_+ \to \Rb$ is (closed) convex
%and lower semi-continuous 
with $f(1) = 0$, and $\lambda$ is a dominating measure (e.g., Lebesgue).
% , and $p(x)$ and $p(y)$ are the marginal densities of $p(x,y)$.
%GZ: didn't we define it already?!
%and recall that $p(x)$ and $p(y)$ are the marginal densities of $p(x,y)$, whereas 
% fixed dominating measure 
%$p(x, y)$ denotes the density function of $p(x,y)$, and $p(x)$ and $p(y)$ denotes the density function of $P_X$ and $P_Y$, respectively. 
\end{definition}

% The $f$-mutual information has been proposed in \citet{csiszar1967information}. The current form as in Definition \ref{def:fMI} can be found in \citet{verdu2015alpha} as $\alpha$-mutual information where the $f$-divergence is R\'{e}nyi $\alpha$-divergence \citep{renyi1961measures} and in  \citet{esposito2020robust} for general $f$-divergence 
%(\red{S: Not clear. Please revise. }).

%Common choices of $f$ can be found in Table \ref{tbl:choices_f_div} and Table \ref{tbl:choices_f_div_app} (Appendix~\ref{app:theory}), which we will discuss in more details in \S~\ref{sec:design}. 

%%% alpha-div is a special case of f-div. no need to discuss if not used later.
%Another generalization of Shannon's mutual information \citep{shannon1948mathematical} is $\alpha$-mutual information \citep{verdu2015alpha}, where the $f$-divergence is replaced with R\'{e}nyi $\alpha$-divergence \citep{renyi1961measures}. These two generalizations are orthogonal and the application of $\alpha$-mutual information is beyond the scope of our paper. 
%It can be shown that the 
Note that the $f$-MI is essentially the $f$-divergence between the joint distribution and the product of marginal distributions. It is well-known that the $f$-MI is non-negative and symmetric. Moreover, provided that $f$ is strictly convex, $I_f(X; Y) = 0$ iff $X$ and $Y$ are independent \citep{ali1966general}. 

% \begin{prop}\label{pro:fmi}
% The $f$-mutual information satisfies non-negativity $I_f(X; Y) \geq 0$ and symmetry $I_f(Y; X) = I_f(X; Y)$. Moreover, if $f$ is strictly convex and $f(1) = 0$, then $I_f(X; Y) = 0 \Longleftrightarrow$ X and Y are independent. 
% \end{prop}
% \begin{proof}
% From Jensen's inequality, we always have $I_f(X; Y) \geq 0.$ The second part can be found in e.g.~\cite{esposito2020robust}.   
% \end{proof}
%When $X$ and $Y$ have a high dimension, it is quite challenging to estimate the $f$-divergence directly. 
%\red{feature space}
%\blue{We study the $f$-MI between two distributions in the feature space. } 

Since it is challenging to provide an accurate estimation of the $f$-divergences in high dimensions, \cite{nguyen2010estimating} used the convex conjugate as a lower bound for the $f$-divergences. With this result we can lower bound $I_f(X;Y)$ as follows:
\begin{align}\label{eq:dual_form_fMI}
%\begin{split}
\teal{I_f(X;Y) \geq} \sup_{s\in \Fc} \Big(\underset{(X, Y) \sim p(x,y)}{\Eb} s(X, Y) - \underset{(X, Y)\sim p(x)p(y)}{\Eb} f^*\circ s(X, Y) \Big),
%\end{split}
\end{align}
where $p(x,y)$ denotes the joint density, $p(x)p(y)$ stands for the product of marginal densities, and the symbol $\circ$ denotes function composition.
The function $f^*(t) := \sup_{x\in \Rb_+}(x t - f(x))$ is known as the convex conjugate\footnote{More precisely, this is the monotone convex conjugate since we restrict the domain of $f$ to $\Rb_+$.} of $f$ and is \emph{monotonically increasing}, and $s(\cdot)$ belongs to $\Fc$, a class of functions on $(x, y)$ that we can parameterize.
%a class of functions of $s$ with 
%$s: \supp(p(\cdot)) \times \supp(p(\cdot)) \to \dom f^*.$
%Throughout the paper we use $i_f(X;Y)$ to denote the maximum of \cref{eq:dual_form_fMI}.
%a subset of measurable functions. If $\Tc$ is the set of all measurable functions then the inequality becomes tight. 
Using results in \cite{nguyen2010estimating}, one can show that  \cref{eq:dual_form_fMI} is equal to $I_f(X;Y)$ if there exists $s_\star \in \Fc$ such that for any $(x, y) \in \supp(p(x)) \times \supp(p(y))$, where $\supp(\cdot)$  denotes the support of a distribution, we have:
\begin{align}\label{eq:condition_T}
s_\star(x,y) = f'\left(\frac{p(x, y)}{p(x)p(y)}\right).
\end{align}
\teal{In other words, plugging the optimal $s_\star(x,y)$ into \cref{eq:dual_form_fMI} we obtain equality.}
In \Cref{tbl:choices_f_div} we list common choices of $f$-divergences, their conjugates, and the derivatives. We also include the composition $f^* \circ f'$ for later purposes (see Theorem \ref{thm:uniformity} below).%} %\blue{modify Table 1}
%in particular, if $\Tc$ comprises of all (measurable) functions.
%\blue{We observe that the dual problem is structurally similar to InfoNCE loss, which motivates us to link $f$-MI with contrastive learning. }
% \red{Later when introducing the contrastive learning we will show that the two terms in $i_f(X;Y)$ have intuitive explanations for representation learning. }
% GZ: don't be repetitive, don't show off everywhere

% The marginals $P_X$ and $P_Y$ are data distributions, which should be the same. Sample pairs from $P_X \otimes P_Y$ are not correlated with each other and we call them \emph{negative pairs}. 

% Our goal is to maximize the lower bound of the $f$-MI in \cref{eq:dual_form_fMI}, and we can interpret the r.h.s.~of \cref{eq:dual_form_fMI} as follows:
% \begin{itemize}
% \item The first term $\mathbb{E}_{(x, z) \sim p(x,y)}[T(x,z)]$ is an expectation over positive pairs from the joint distribution $p(x,y)$, maximizing it means that we want similar samples to have close feature embeddings;
% \item The second term is an expectation over negative pairs from the product distribution $P_X \otimes P_Y$, and minimizing $\mathbb{E}_{(x, y)\sim \squares}[f^*(T(x,y))]$ means that we want uncorrelated samples to have  feature embeddings that are far away, since $f^*$ is a monotonically increasing function. 
% \end{itemize}

\begin{table*}
\setlength\tabcolsep{18pt}
\caption{Common choices of $f$-divergences. KL: Kullback--Leibler; JS: Jensen--Shannon; SH: Squared Hellinger; VLC: Vincze--Le Cam \protect\citep{le2012asymptotic}. For JS, we define $\varphi(u) = - (u+1) \log \protect\frac{1+u}{2}+u \log u$. The Tsallis-$\alpha$ divergence is defined in \protect\cite{tsallis1988possible}. See \protect\Cref{app:f-div} for more details.}
%{\color{blue}{guojun: keep the format of the tables in accordance with ICLR 2022; change the notations to ICLR format}}
\label{tbl:choices_f_div}
\begin{center}
%\footnotesize
	\begin{tabular}{llllll}
	\toprule
	\centering
	{\bf Divergence} & $f(u)$  & $f^*(t)$ & $f'(u)$ & $f^*\circ f'(u)$ \\
	\midrule \midrule
	KL
	& $u \log u$
	& $\exp(t-1)$
	& $\log u + 1$
	& $u$
	\\
	JS
	& $\varphi(u)$
	& $- \log(2-e^t)$
	& $\log 2 + \log \frac{u}{1+u}$
	& $-\log \frac{2}{1+u}$
	
	\\
	Pearson $\chi^2$ 
	& $(u-1)^2$ 
	& $t^2/4 + t$
	& $2(u-1)$
	& $u^2 - 1$
	\\
	SH 
	& $(\sqrt{u} - 1)^2$
	& $\frac{t}{1 - t}$
	& $1 - u^{-1/2}$
	& $u^{1/2} - 1$ 
	\\
	Tsallis-$\alpha$
	& $\tfrac{u^\alpha}{\alpha - 1}$
	& $(\tfrac{\alpha - 1}{\alpha}t)^{\tfrac{\alpha}{\alpha-1}}$
	& $\frac{\alpha}{\alpha - 1} u^{\alpha - 1}$
	& $u^\alpha$
	\\
	VLC
	& $\frac{(u-1)^2}{u + 1}$ 
	& $4 - t - 4\sqrt{1-t}$
	& $1 - \frac{4}{(u+1)^2}$
	& $3 - \frac{4}{u+1}$
	\\
	\bottomrule
	\end{tabular}
\end{center}
%\vspace{-1.0em}
\end{table*}
%\vspace{-0.5em}
\section{$f$-MICL}\label{sec:design} 
%\vspace{-0.5em}

With the introduction of the more general $f$-MI we now proceed to the design of the objective and similarity function. Then we will analyze the property of the proposed framework and compare it with some existing InfoNCE-based benchmarks. 

\subsection{$f$-MICL objective}
Contrastive learning is a popular \emph{self-supervised} method for representation learning. In contrastive learning, we expect similar sample pairs to be close to each other in the embedding space, while dissimilar pairs to be far apart. Based on the $f$-MI introduced in \S~\ref{sec:prem}, we propose a general framework for contrastive learning, coined as $f$-MICL.  
%Furthermore, we will characterize the properties of alignment and uniformity theoretically for general $f$-divergences. 

%Assume that the data distribution $p(x,y)$ is symmetric w.r.t. the random variables $x$ and $y$. Hence, the resultant two marginals both follow the data distribution $p_d$ \citep{wang2020understanding}.

% Following \cite{ChenKNH20}, \red{we design the function $T$ in \cref{eq:dual_form_fMI}} as follows:
% \begin{align}\label{eq:embedding}
% T(x, y) := s(x^g, y^g),\, \|x^g\| = 1\,\mbox{ for any sample }x.
% \end{align}
%Note that we used the shorthand $x^g := g(x)$ and $y^g := g(y)$.
%The function $g$ produces a $d$-dimensional normalized feature encoding on the hypersphere $\Sb^{d-1}$ and $s$ is a similarity function that measures the similarity between two embeddings $x^g$ and $y^g$. 
%Our goal is to maximize the lower bound of the $f$-MI in \cref{eq:dual_form_fMI} for encouraging the contrastiveness between positive and negative pairs. 
%GZ: ppl rarely say data representations or data embeddings
%GZ: this sounds super weird
We denote $g: \Xc\to \Sb^{d-1}$ as the feature encoder (usually constructed by a neural network) from the input space $\Xc$ to the hypersphere, and we use the shorthands $x^g := g(x)$ and $y^g := g(y)$ to represent the feature embeddings. The notation $p_{d}$ stands for the data distribution, and $\squares := p_d\otimes p_d$ means its self product (\teal{product of marginals, e.g., pairs of images}). We denote $p_+$ as the distribution of \emph{positive pairs}, \ie, two samples with similar feature embeddings (\teal{joint distribution, e.g., the same image with different data augmentation}). 
Using the lower bound of the $f$-MI in \cref{eq:dual_form_fMI}, we have the  general $f$-MICL objective as follows:
% We can compute the $f$-MI lower bound from \eqref{eq:dual_form_fMI} and arrive at the general $f$-MICL objective:
%In addition, we use $s(\cdot, \cdot)$ to represent the similarity function between two embeddings. 
%With the above definition of $T(x, y)$ we can rewrite the maximand of \cref{eq:dual_form_fMI} as:
\begin{align}\label{eq:objective}
%\begin{split}
%\sup_{g\in \Gc, k\in \Kc} i_f(X;Y) &:= 
\vio{\sup_{s\in \Fc}}~~
\mathbb{E}_{(x, y) \sim \ppos} s(x^g,y^g) - \mathbb{E}_{(x, y)\sim \squares} f^*\circ s(x^g,y^g),  
%\end{split}
\end{align}
%where we maximize over the feature embedding $g$ and the similarity function $s$ over function classes  $\Gc$ and $\Sc$ respectively. 
where $s(x^g,y^g)$ can be understood as the similarity measurement between two feature embeddings in the context of contrastive learning.
%\red{where $s(x^g,y^g)$ has semantic meaning of measuring similarity between two feature embeddings in the context of contrastive learning. }
Essentially, we are studying the variational lower bound \cref{eq:dual_form_fMI} in the \emph{feature space}, with the feature embeddings learnable. %\blue{Thus we have arrived our objective of $f$-MI applying to contrastive learning.}
We can treat the first term as the similarity score between \emph{positive pairs} with similar feature embeddings, and the second term as the similarity score between two random samples, a.k.a.~\emph{negative pairs}. As $f^*$ is an increasing function, maximizing the $f$-MI is equivalent to simultaneously maximizing the similarity between positive pairs and minimizing the similarity between negative pairs.

With \cref{eq:objective}
%To conclude, 
we have answered the first question: there are a spectrum of $f$-MICL objectives that can be applied in contrastive learning by using different $f$ functions. We will   discuss  how to choose  the best $f$ empirically in \S\ref{sec:exp}.
%We will further show in \S~\ref{sec:infonce} that InfoNCE is closely related to the $f$-MICL objective. 

%To obtain the variational lower bound of $f$-MI, we still need to maximize over the similarity function $s$ and the feature embedding $g$ in \eqref{eq:objective}. In the next subsection we will study how to find a proper similarity function $s$. After we find a good similarity function, the rest is to train the feature encoder $g$ as a neural network with our $f$-MICL algorithm. 

% In this section we give a principled way to construct our objective based on the lower bound in \cref{eq:dual_form_fMI}. We study how to choose the similarity function between pairs of samples and which sampling method to use. These choices are justified with our theory. Specifically, we show that under mild assumptions we can retrieve uniformity and alignment as in \cite{wang2020understanding}. 
%\vspace{-0.3em}
\subsection{$f$-Gaussian similarity}\label{subsec:sim}
%\vspace{-0.2em}
%\newcommand{\Gc}{\mathcal{G}}

% Next, we want to answer the second question: 
% how to design the similarity function $s(x^g,y^g)$? Previous works usually  adopts the popular cosine similarity function $s(x^g,y^g) = (x^g)^\top y^g/\|x^g\|\|y^g\|$,
% Although it shows promising performance in practice, it is not clear if this $s(\cdot)$ obtain the variational lower bound in eq. \ref{eq:objective}.  

\begin{figure}
    \centering
    \includegraphics[width=0.8\textwidth]{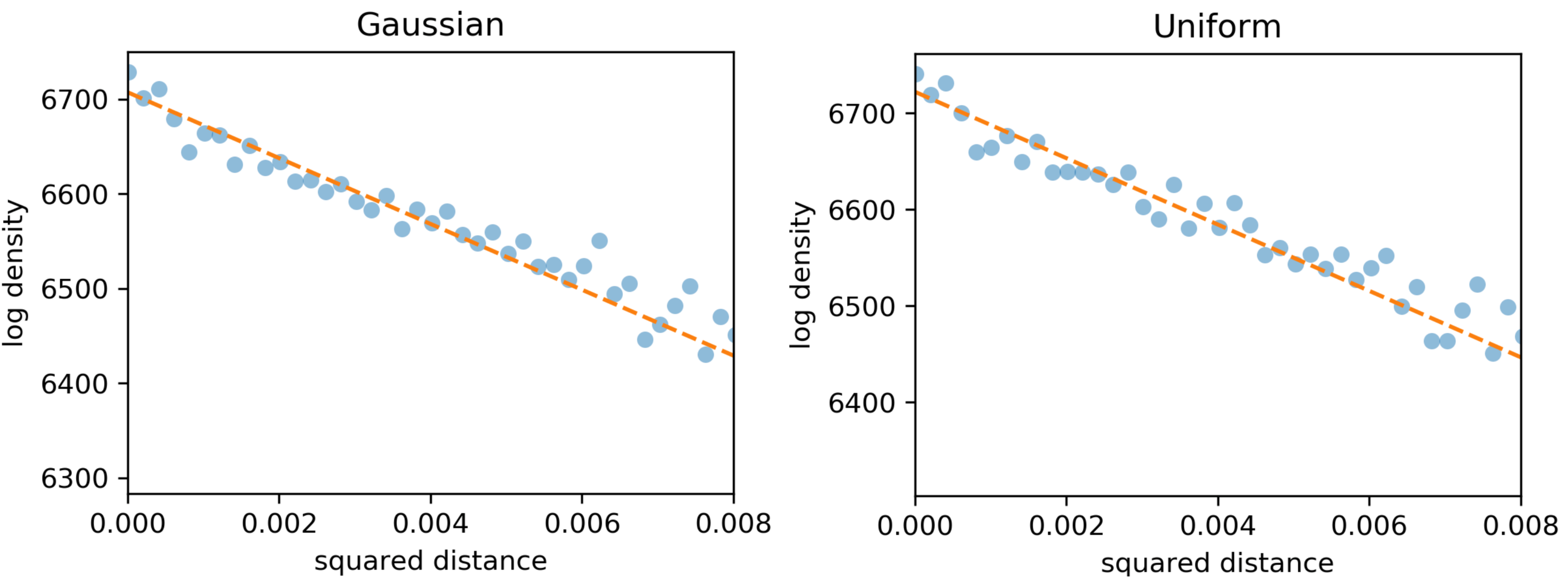}
    %\vspace{-0.2em}
    \caption{Experiment for verifying Assumption \ref{assmp:joint_vMF}. Here we draw the relation between the squared distances $\|x^g - y^g\|^2$ and the averaged log likelihood $\log p_g$, with $\log p_g$ estimated by the flow model RealNVP \citep{dinh2016density}. \vio{({\bf left}) Gaussian prior; {\bf (right)} Uniform prior.}
    The features are learned by SimCLR trained on CIFAR-10. See more details in Appendix \ref{sec:add_exp}.
    }
    \label{fig:gaussian}
\end{figure}
%\vspace{-0.2em}

Previous works in constrastive learning usually adopt some heuristic similarity function such as the cosine similarity function $s(x^g,y^g) = x^g \cdot y^g$. Although it shows promising performance in practice, our second question is that, can we provide a better similarity function than the popular cosine similarity?

%To address this, we recall eq.
Note that 
\cref{eq:condition_T} is a natural choice of $s(\cdot)$ from the perspective of deriving the $f$-MI.
%the definition of $f$-MI. 
In the context of contrastive learning, by denoting the density functions of the marginal feature distributions as $p_g(x^g)$ and $p_g(y^g)$, and the density of the joint feature distribution as $p_g(x^g, y^g)$, from \cref{eq:condition_T} we have an optimal similarity function as follows:
% the optimal solution of $s$, similarly as \cref{eq:condition_T}:

%In this subsection, we provide a principled way to design the similarity function. To our best knowledge, there has been no theoretical study on the choice of similarity functions, and most existing contrastive learning methods adopt a pre-designed similarity function, such as cosine similarity \citep[e.g.][]{ChenKNH20, he2020momentum} and bilinear functions \citep{oord2018representation}. 

%Suppose $(x, y) \sim \ppos$, then we denote the density functions of the marginal feature distributions as $p_g(x^g)$ and $p_g(y^g)$, and the density of the joint feature distribution as $p_g(x^g, y^g)$. Comparing \cref{eq:dual_form_fMI} and \cref{eq:objective}, we can write down the optimal solution of $s$, similarly as \cref{eq:condition_T}:
%We remind the reader of the following result \blue{from \cref{eq:condition_T} in the feature space}:

% Note that the similarity function $k$ is fixed and we do not optimize it. One natural question is: how do we choose the similarity function? It is a common choice that $k(g(x), y^g) = x^g \cdot y^g$ is the cosine similarity (if the feature $g$ is normalized). However, is it the optimal choice? Our next theorem gives answer no. For simplicity we

\begin{restatable}[{\eg, \citealt[Lemma 1]{nguyen2010estimating}}]{lem}{optimalsim}\label{lem:optimal_sim}
Suppose $f$ is differentiable, and the embedding function $g$ is fixed. 
%the function class $\Kc = \{k: \range \, k \subseteq \dom f^*\}$.
% Suppose that the range of the function $k \subseteq \dom f^*$ and $k$ is unconstrained. Assume also that $f$ is differentiable. 
%Given the encoder function $g$, the optimal choice of $k$ to maximize $i_f(X;Y)$ in \cref{eq:objective} is as follows:
The following similarity function $s_\star$ maximizes \cref{eq:objective}:
\begin{align}\label{eq:optimal_similarity1}
s_\star(x^g, y^g) = f'\left( \frac{p_g(x^g, y^g)}{p_g(x^g)p_g(y^g)} \right).
\end{align}
\end{restatable}

% Comparing \cref{eq:optimal_similarity1} with \cref{eq:objective}, we realize that 
Obviously,
the optimal $s_\star$ in fact induces the $f$-MI on the feature space, which is a low bound of the original $f$-MI. 
Although \cref{eq:optimal_similarity1} provides an optimal similarity function, it nevertheless depends on the unknown density functions. %In fact, if we knew the feature embedding beforehand, we  
How can we implement \cref{eq:optimal_similarity1} in practice? 
%One possible approach is to assume that we ``knew'' the form of the feature density. 
Among various known density functions, it is natural to choose a typical kernel function for structured data for validation \citep{balcan2008theory,powell87,murphy2012machine}, e.g., the Gaussian kernel.
%\red{regarding the joint feature distribution:}

% However, in order for a practical framework we need information about the joint distribution $p_g(x_g, y_g)$. Even though this joint distribution includes the feature embedding we want to learn in the end, we can make mild assumptions for it first in order to obtain a relatively good similarity function with \cref{eq:optimal_similarity1}. The assumption we make is:

\begin{assump}[\textbf{Gaussian kernel}]\label{assmp:joint_vMF}
The joint feature density (wrt the uniform distribution over the hypersphere $\Sb^{d-1}$) is proportional to a Gaussian kernel, namely 
%i.e., $$p_g(x^g, y^g) \propto \varphi(\|x^g - y^g\|^2)\mbox{ for a real-valued function }\varphi.$$ 
\begin{align}%\label{eq:joint_vMF}
p_g(x^g, y^g) \propto G_\sigma(\|x^g - y^g\|^2) = \mu \, \exp\left(-\tfrac{\|x^g - y^g\|^2}{2\sigma^2}\right),\nonumber
\end{align}
where $\mu := \exp(\tfrac{1}{\sigma^2}) \tfrac{C}{c^2}$ is a constant that we determine below.
%GZ: we are not doing other kernels
%GZ: no it's not a normalization constant. 
\end{assump}
Since $x^g, y^g \in \Sb^{d-1}$ have unit (Euclidean) norm, we have 
\begin{align}
\label{eq:vM-F}
    p_g(x^g, y^g) \propto \exp( \tfrac{x^g \cdot y^g}{\sigma^2} ),
\end{align}
which belongs to the von Mises-Fisher bivariate distribution \citep[Eq. 2.11]{Mardia75}. 
It is clear that the marginals of $p_g$ are uniform. Indeed, for any orthogonal matrix $Q$, with $\tfrac{1}{C} := \Eb_{(x^g,y^g) \sim \Sb^{d-1}\times \Sb^{d-1}} \exp(\tfrac{x^g \cdot y^g}{\sigma^2} )$ we have
\begin{align}
    p_g(Qx^g) = \Eb_{ y^g \sim \Sb^{d-1}} C \exp(\tfrac{Qx^g \cdot y^g}{\sigma^2} )= \Eb_{ y^g \sim \Sb^{d-1}} C \exp(\tfrac{x^g \cdot Q^\top y^g}{\sigma^2} ) = p_g(x^g) =: c,
\end{align}
where we have used the fact that the only invariant distribution on $\Sb^{d-1}$ wrt the orthogonal group is the uniform distribution.
Similarly, we have $p_g(y^g) \equiv c$ (where $c$ is the reciprocal of the surface area of the hypersphere $\Sb^{d-1}$).
The distribution (\ref{eq:vM-F}) has a nice interpretation\footnote{As suggested by the action editor, we may also interpret the distribution (\ref{eq:vM-F}) as a ``copula,'' i.e., a joint density on $\Sb^{d-1} \times \Sb^{d-1}$ with uniform marginals. Note that the conventional notion of copula replaces the hypersphere $\Sb^{d-1}$ with the unit interval $[0,1]$. More generally, we could consider the ``copula'' $p_g(x^g, y^g) \propto h(x^g \cdot y^g)$ for an increasing function $h$, to capture other types of correlation between the two views $x^g$ and $y^g$. } in terms of maximum entropy \citep{Mardia75}, and admits the factorization
\begin{align}
    p_g(x^g, y^g) = p_g(x^g) \cdot p_g(y^g | x^g) = p_g(y^g) \cdot p_g(x^g | y^g),
\end{align}
where the marginals $p_g(x^g)$ and $p_g(y^g)$ are uniform while the conditionals $p_g(y^g|x^g)$ and $p_g(x^g | y^g)$ again belong to the von Mises-Fisher distribution.

\vio{Combining \cref{eq:optimal_similarity1} and Assumption~\ref{assmp:joint_vMF}}, we can write the similarity function with the Gaussian kernel as follows: 
\begin{align}\label{eq:optimal_similarity}
\textstyle
s_f(x^g, y^g) = f' \circ G_\sigma(\|x^g - y^g\|^2).
\end{align}
% \vio{where we have used the following lemma:
% \begin{restatable}[\textbf{uniform marginals}]{lem}{uniform}\label{prop:uniform}
% Under Assumption \ref{assmp:joint_vMF}, the marginal distributions $p_g(x^g)$ and $p_g(y^g)$ are uniform on the hypersphere $\Sb^{d-1}$, with $d$ the feature dimension. 
% \end{restatable}
% }
% \vio{The detailed proof can be seen in \Cref{sec:uniform_proof}, and we provide a sketch here. Note that we choose $x^g$ and $y^g$ to be both normalized such that when we marginalize either of them, the integral would be a constant due to rotational invariance. For example, there exists constant $C$ such that:
% \begin{align}
% p(x^g) = \int_{\Sb^{d-1}} C \, \exp\left(-\frac{\|x^g - y^g\|^2}{2\sigma^2}\right) dy^g = C \int_{\Sb^{d-1}} \, \exp\left(-\frac{1 - x^g \cdot y^g}{\sigma^2}\right) dy^g.
% \end{align}
% Whatever the value of $x^g$ is, the marginal $p(x^g)$ remains the same, since any two $x^g$'s on a unit sphere are equal.}
\vio{As noted above, the product of marginals $p_g(x^g)p_g(y^g)$ is a constant, which has been absorbed into our definition of $G_\sigma$, see $\mu$ in Assumption~\ref{assmp:joint_vMF}.}
We observe that $s_f(\cdot)$ depends on the choice of $f$ as well,
%from which we identify that 
%$s^*(\cdot)$ is subjective to the choice of $f$ as well, 
thus we call it \textbf{$f$-Gaussian similarity}. As a result, we have
provided a new way to design the similarity function, again from the $f$-MI perspective.
% answered the second question: it is possible to design the similarity function, again from the $f$-MI perspective. 

\paragraph{Verifying Assumption \ref{assmp:joint_vMF}:}  
One may question that Assumption \ref{assmp:joint_vMF} can be too strong for practical usage. For example, replacing the Gaussian kernel $G_\sigma$ with any other decreasing function would also provide a valid assumption. However, we found that among several popular choices only the Gaussian kernel works well in practice. Also,  we can  empirically verify that Assumption \ref{assmp:joint_vMF} approximately holds.
%\textbf{Verifying Assumption \ref{assmp:joint_vMF}:} 
To this end, it is sufficient to check
%here we empirically check 
whether the log density, \ie, $\log p_g(x^g, y^g)$, is linear with the distance between each positive pair, \ie,  $\|x^g - y^g\|^2$. In \Cref{fig:gaussian}, we use the flow-based model RealNVP \citep{dinh2016density}\footnote{\teal{RealNVP applies real-valued non-volume preserving transformation for log-likelihood computation.}} to estimate the log density \blue{with a Gaussian prior and a uniform prior}, 
%We generate positive sample pairs via data augmentation (\Cref{fig:f-MICL-intro})
and learn the  feature encoder $g$ from SimCLR \citep{ChenKNH20}.
% and plot the relation between $\log p_g$ and $\|x^g - y^g\|^2$, where the feature encoder $g$ is learned from SimCLR \citep{ChenKNH20}. 
We observe that the linear relationship approximately holds for CIFAR-10 \footnote{The linear relationship in \Cref{fig:gaussian} might also depend on the data, i.e., the CIFAR-10 dataset here. In practice, other customized datasets might require additional verification. }. 

We will empirically compare our $f$-Gaussian similarity with the cosine similarity in \S\ref{sec:exp}.

\subsection{Implementation}
\label{subsec:imp}
With our designed $f$-Gaussian similarity $s_f$ we now have an implementable $f$-MICL objective in \cref{eq:objective}. 
Bringing the $f$-Gaussian Similarity $s_f$ in \cref{eq:optimal_similarity} into our objective \cref{eq:objective} we have a specific $f$-MICL objective:
\begin{align}\label{eq:objective_gaussian}
\underset{(x, y) \sim \ppos}{\Eb} s_f(x^g, y^g) - \underset{(x, y)\sim \squares}{\Eb} f^* \circ s_f(x^g, y^g).
\end{align}
Given a batch of $N$ samples, 
its empirical estimation is as follows:
\begin{align}\label{eq:objective_sample}
%\widehat{i}_f(X;Y) = 
%\textstyle
\frac{1}{N}\overset{N}{\underset{i=1}{\sum}} s_f(x_i^g, y_i^g) - \frac{1}{N(N -1)}\underset{i\neq j}{\sum} f^* \circ s_f(x_i^g, x_j^g),
\end{align}
where $x_i$ and $y_i$ are two types of data augmentation of the $i$-th sample, and $x_i$ and $x_j$ are different samples with \vio{independently sampled} data augmentations.
%(see \Cref{fig:f-MICL-intro} for illustration). 
%This sampling method follows InfoNCE \cite{ChenKNH20}. 

With the $f$-MICL objective in \cref{eq:objective_sample} we propose our algorithm for contrastive learning in Algorithm~\ref{alg:f-MICL}. To balance the two terms in our objective, we additionally include a weighting parameter $\alpha$ in front of the second term (which also absorbs the parameter $\mu$ in $G_\sigma$). This change can still be incorporated within our $f$-MICL framework, as we show in \Cref{sec:weighting}. %\blue{check this part}
%To illustrate the implementation of our algorithm, we also show the network architecture in Figure \ref{fig:f-MICL-intro}. 
% We list some common choices of $f$-divergences in Table \ref{tbl:choices_f_div}. A more detailed version can be found in Table \ref{tbl:choices_f_div_app} in Appendix \ref{app:f-div}.  
Figure \ref{fig:f-MICL-intro} gives a high-level summary of our $f$-MICL framework. Given a batch of samples (\eg, images), we generate \emph{positive pairs} via data augmentation and \emph{negative pairs} using other augmented samples in the same batch. 
This sampling method follows SimCLR \citep{ChenKNH20}. 

\begin{algorithm}[t]
%\SetAlgoLined
\DontPrintSemicolon
\KwIn{batch size $N$, function $f$, weighting parameter $\alpha$, constant $\mu$ (in $G_\sigma$), variance $\sigma^2$}
%, boolean variable \texttt{symmetric} %(\red{Polish it based on Algorithm 2.}) 

\For{each sampled mini-batch $\{z_i\}_{i=1}^N$}{
\For{$k$ in $1, \dots, N$}{
randomly sample two augmentation functions $t_1$ and $ t_2$ \\
$y_k \leftarrow t_1(z_k)$, $x_k \leftarrow t_2(z_k)$ \\
}
define $s_f(x^g, y^g) = f'\circ G_\sigma(\|x^g - y^g\|^2)$ \\
compute ${-\cal L}$ as
%\vspace{-0.5em}
\begin{align}
%i_f &= 
\nonumber
\frac{1}{N}\sum_{i=1}^N s_f(x_i^g, y_i^g) \nonumber - \frac{\alpha}{N(N -1)}\sum_{i\neq j}f^* \circ s_f(x_i^g, x_j^g) 
\end{align}
\\
% \If{\texttt{symmetric}}{$i_f \leftarrow i_f - \frac{\alpha}{N(N-1)}\sum_{i\neq j}f^* \circ f'\left( C G_\sigma(\|x_i^g - y_i^g\|)\right)$}
update $g$ by minimizing ${\cal L}$
}
\caption{%$f$-mutual information contrastive learning 
$f$-MICL}
\label{alg:f-MICL}
%\vspace{-0.3em}
\end{algorithm}

\begin{figure*}
    \centering
    \includegraphics[width=0.9\textwidth]{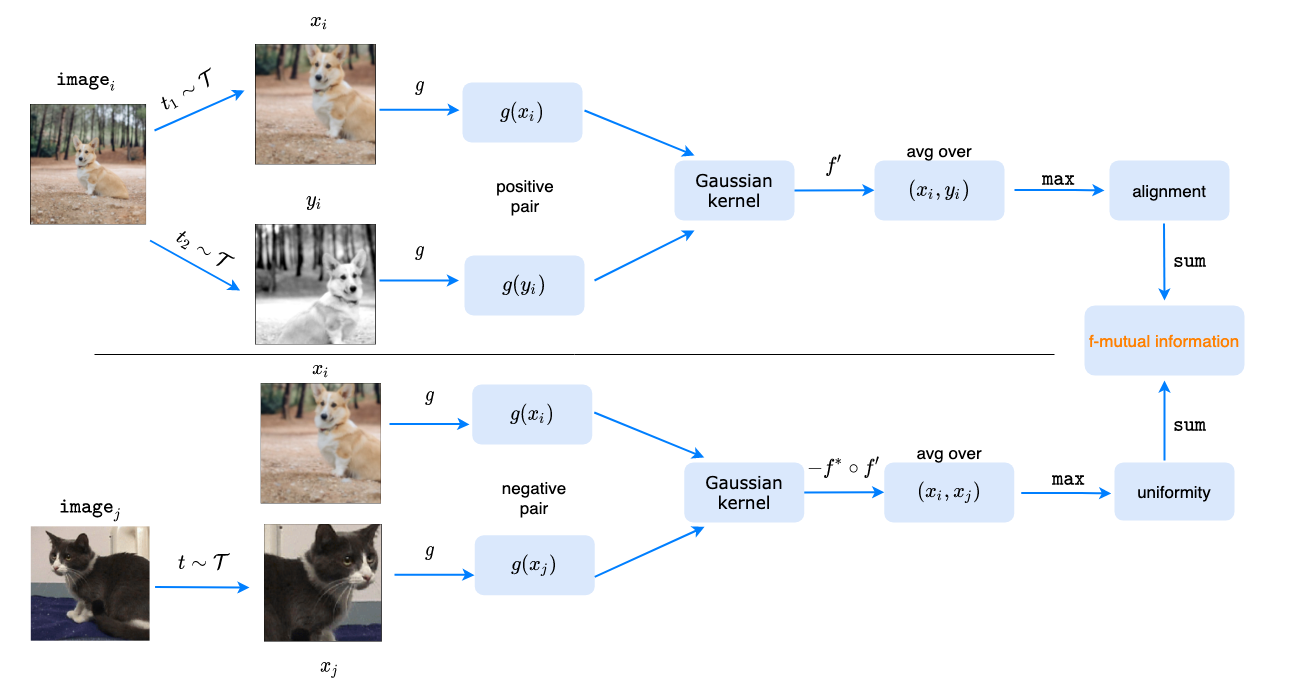}
    %\vspace{-0.5em}
    \caption{Network architecture of $f$-MICL. $\mathtt{image}_i$: the $i^{\rm th}$ image in the current batch; $f$: the function used in the $f$-mutual information (\S\ref{sec:prem}); $g$: feature embedding; $t$, $t_1$, $t_2$: augmentation functions drawn from the same family  $\Tc$ of augmentations; $f'$: the derivative; $f^*$: the Fenchel conjugate. The symbol $\circ$ denotes the function composition. The sum of the two terms gives the variational lower bound of $f$-mutual information. \teal{$x_i$ and $y_i$ are two types of data augmentation of the $i$-th sample, and $x_i$ and $x_j$ are different samples with \vio{independently sampled} data augmentations.} \vio{\texttt{max} stands for maximization.} See \cref{eq:objective_sample} for more details.}
    \label{fig:f-MICL-intro}
    %\vspace{-1.0em}
\end{figure*}

\subsection{$f$-MICL family}
\label{sec:infonce}
%\vspace{-0.3em}
In this section, we will deepen the understanding of $f$-MICL by drawing connections with  some popular constrastive learning methods. 

%In particular, we will identify that InfoNCE as well as some widely used benchmarks are in fact special examples of $f$-MICL.  

% Finally, we discuss (1) how $f$-MICL provides key insights to understand InfoNCE-based objectives, (2) why $f$-MICL is a natural generalization of InfoNCE. 

\noindent \textbf{Connection with InfoNCE:} 
Firstly, we show that InfoNCE is an upper bound of $f$-MICL. 
Recall  our $f$-MICL objective in  \cref{eq:objective}, and the popular InfoNCE objective $\mathcal{L}_\text{InfoNCE}$  as follows (here we take the maximization) \citep{oord2018representation}:
\begin{align}
    \mathbb{E}_{(x, y) \sim \ppos} s(x^g,y^g) - \mathbb{E}_{ x\sim p_d} \log \mathbb{E}_{ y\sim p_d} \exp (s(x^g,y^g)).
\end{align}
%How do we link this objective to the InfoNCE loss?
%We now draw connections between these two.
%Naturally, 
Consider that we perform a Donsker-Varadhan (DV) shift transformation $v$ \citep{dv,tsai2021self} from  \cref{eq:objective} such that by taking the maximum over the transformation we have:
\begin{align}
    &\sup_{v\in \mathbb{R}} \Big( \mathbb{E}_{(x, y) \sim \ppos} s(x^g,y^g) - v   - \mathbb{E}_{(x, y)\sim \squares} f^*\circ (s(x^g,y^g)-v)   \Big).
    \label{eq:dv}
\end{align}
In practice, such a shift transformation can be approximated by a scaling factor  ($\alpha$ in Algorithm \ref{alg:f-MICL}) such that  \cref{eq:objective} and \cref{eq:dv} are equivalent.
Given that $f$ is the KL divergence, thus $f(u) = u \log u$ and $f^*(t) = \exp(t-1)$ from Table \ref{tbl:choices_f_div}, 
the maximizer of $v$ in \cref{eq:dv} occurs at $v_\star = \log (\mathbb{E}_{(x, y)\sim \squares} s(x^g,y^g))-1$). With $v_\star$,  \cref{eq:dv} can be written as follows:
%we arrive at:
\begin{align}
    \mathbb{E}_{(x, y) \sim \ppos} s(x^g,y^g) - \log \mathbb{E}_{(x, y)\sim \squares} \exp (s(x^g,y^g)).
    \label{eq:infonce}
\end{align}
According to Jensen's inequality we have
\begin{align}
    &\mathbb{E}_{ x\sim p(\cdot)} \log \mathbb{E}_{ y\sim p(\cdot)} \exp (s(x^g,y^g)) \leq \log \mathbb{E}_{(x, y)\sim \squares} \exp (s(x^g,y^g)).
\end{align}
Therefore,
\begin{align}
    &\mathbb{E}_{(x, y) \sim \ppos} s(x^g,y^g) - \log \mathbb{E}_{(x, y)\sim \squares} \exp (s(x^g,y^g))
    \leq \mathcal{L}_{\text{InfoNCE}}.
\end{align}

\begin{figure}%[13]{r}{0.26\textwidth}
  \begin{center}
    \includegraphics[width=0.4\textwidth]{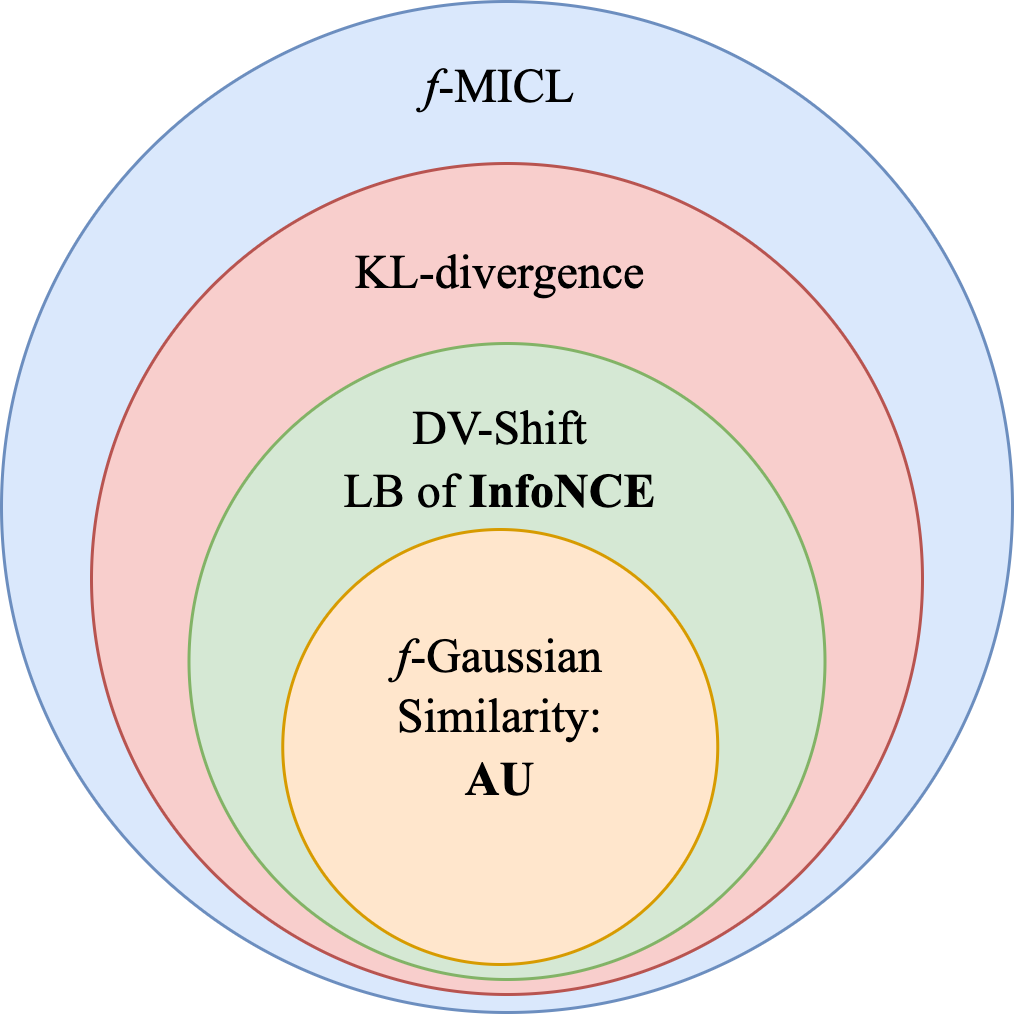}
  \end{center}
  %\vspace*{-7mm}
  %\vspace{-2em}
  \caption{$f$-MICL generalizes InfoNCE-based objectives.}
  \label{fig:relation}
  %\vspace{-1em}
\end{figure}

The above transformation shows that the InfoNCE loss is an upper bound of the $f$-MICL objective. In other words, maximizing $f$-MICL can potentially increase the InfoNCE objective.
% immediately guarantees the InfoNCE objective to be large.
%This indicates that $f$-MICL is a tighter lower bound than InfoNCE on mutual information,

%We will empirically prove that $f$-MICL performs better in practice.

%This also provides a different understanding of InfoNCE from the $f$-MI prospective.

%which intuitively help us developing a deeper understanding of InfoNCE from the $f$-MI prospective. 

\textbf{Connection with Alignment and Uniformity (AU) \citep{wang2020understanding}: 
} 
We further show that the Alignment and Uniformity (AU) loss is a special cases of $f$-MICL. 
%Based on the DV-shifted KL-MICL objective in  \cref{eq:infonce}, we observe that :\\
\begin{comment}
(1) Applying the cosine similarity $s(x^g,y^g) = \langle x^g,  y^g \rangle/\|x^g\|\|y^g\|$ with a temperature constant $\tau$, we can recover exactly the SimCLR/MoCo objective:
\begin{align}
    \mathbb{E}_{(x, y) \sim \ppos} \frac{\langle x^g,  y^g \rangle}{\tau\|x^g\|\|y^g\|} - \log \mathbb{E}_{(x, y)\sim \squares} \exp \left(\frac{\langle x^g,  y^g \rangle}{\tau \|x^g\|\|y^g\|}\right).
\end{align}
\end{comment}
\cite{wang2020understanding} shows that InfoNCE approximately aligns positive feature embeddings while encouraging uniformly distributed negative ones. \cite{wang2020understanding} further proposes a new objective which quantifies such properties. Here we show that this new objective is essentially a subclass of the InfoNCE loss under the $f$-MICL framework. Concretely,  applying the $f$-Gaussian similarity function for the KL divergence, we have $f'(u) = \log u +1$ from Table \ref{tbl:choices_f_div} and thus  $s_f (x^g, y^g) = -\|x^g - y^g\|^2$. Using $s_f (x^g, y^g)$ in \cref{eq:infonce} we can recover the AU objective:
\begin{align}
    -\mathbb{E}_{(x, y) \sim \ppos} \|x^g - y^g\|^2 - \log \mathbb{E}_{(x, y)\sim \squares} \left[ \exp (-\|x^g - y^g\|^2) \right].
\end{align}
Note that for KL, this is equivalent to the cosine similarity with a scaling factor: $-\|x^g - y^g\|^2=2x^g \cdot y^g-2$.

\noindent \textbf{Connection with the Spectral Contrastive Loss: }
Here we show that $f$-MICL objective is closely related to the Spectral Contrastive Loss \citep{haochen2021provable}.
Recall our objective:
\begin{align}\label{eq:objective1}
\mathbb{E}_{(x, y) \sim \ppos} s(x^g,y^g) - \mathbb{E}_{(x, y)\sim \squares} f^*\circ s(x^g,y^g),  
\end{align}
where $s_f(x^g, y^g) = f' \circ G_\sigma(\|x^g - y^g\|^2)$. If we choose the Pearson $\chi^2$  divergence, where $f(u) = (u-1)^2, f'(u)=2(u-1), f^*\circ f'(u) = u^2-1$, we have our $\chi^2$-MICL objective:
 \begin{align}\label{eq:chi}
&2\mathbb{E}_{(x, y) \sim \ppos} G_\sigma(\|x^g - y^g\|^2) -\mathbb{E}_{(x, y)\sim \squares} G_\sigma(\|x^g - y^g\|^2)^2 - 3.  
\end{align}
This recovers the spectral contrastive loss exactly if we choose the proper hyperparameter and apply the cosine similarity instead. Thus we generalize the spectral contrastive loss as a special case of $\chi^2$-MICL.

\noindent \textbf{More on AU:} 
%Finally, we want to check whether the alignment and uniformity (AU) property of InfoNCE extends to the general $f$-MICL family. 
Finally, based on our objective in \cref{eq:objective_gaussian} we will show that the alignment and uniformity (AU) property of InfoNCE also extends to the general $f$-MICL family:
(1) Alignment:
In the ideal case, maximizing the first term of \cref{eq:objective_gaussian} would yield $x^g = y^g$ for all $(x, y) \sim \ppos$, \ie, similar sample pairs should have aligned representations. Note that the derivative $f'$ is increasing since $f$ is convex.
(2) Uniformity:
%in Appendix \ref{app:f-div}.  
We demonstrate the uniformity property by minimizing the second term of \cref{eq:objective_gaussian}, or more rigorously and realistically, its empirical version in \cref{eq:objective_sample}.

\begin{restatable}[\textbf{Uniformity}]{thm}{uniformDis}\label{thm:uniformity}
Suppose that the batch size $N$ satisfies $2\leq N \leq d + 1$, with $d$ the dimension of the feature space. If the real function
\begin{align}\label{eq:f-div_property}
h(t) = f^* \circ f'\circ G_\sigma(t) \mbox{ is strictly convex on }[0, 4],
\end{align}
then all minimizers of the second term of \cref{eq:objective_sample}, \ie,
$\sum_{i\neq j}f^* \circ s_f(x_i^g, x_j^g)$, 
satisfy the following condition: the feature representations of all samples  
are distributed uniformly on the unit hypersphere $\Sb^{d-1}$. 
\end{restatable}

% \begin{wrapfigure}[10]{r}{0.22\textwidth}
% \vspace{-8mm}
%   \begin{center}
%     \includegraphics[width=0.2\textwidth]{images/tetrahedron.pdf}
%   \end{center}
%   \vspace{-10mm}
%   \caption{A regular simplex on a hypersphere.}
%   \label{fig:tetrahedron}
% \end{wrapfigure}

In Theorem \ref{thm:uniformity}, the assumption $N\leq d + 1$ is always satisfied in our experiments in \S \ref{sec:exp}. For instance, on CIFAR-10 we chose $N = d = 512$. %\red{table 1}
Also, we claim that the samples are  ``distributed uniformly'' if the feature vectors form a regular simplex, %(see Figure \ref{fig:tetrahedron})
 and thus the distances between all sample pairs are the same. Although minimizing the negative term gives uniformity, the positive term is also needed for aligning similar pairs, as we observe in \S \ref{sec:exp}. This implies the \emph{tradeoff} between alignment and uniformity.

% As shown in \cite{wang2020understanding}, uniformity would lead to better separability among different classes.
In fact, \cref{eq:f-div_property} provides us guidance to select proper $f$-divergences for uniformity. In Table \ref{tbl:choices_f_div}, we list some common choices of $f$-divergences. By inspecting the last column and using the definition of $G_\sigma$, we can easily verify that they all satisfy \cref{eq:f-div_property}. 
%However, there exist exceptions.
% Theorem \ref{thm:uniformity} shows that for $f$-divergences listed in Table \ref{tbl:choices_f_div}, the negative term implies uniformity. 
However, this is not true for all $f$-divergences.
In Appendix \ref{app:f-div}
we also provide some counterexamples that violate \cref{eq:f-div_property} and thus Theorem \ref{thm:uniformity}, such as the Reversed Kullback--Leibler (RKL) and the Neynman $\chi^2$ divergences. 
Experimentally, we found that these divergences generally result in feature collapse (\ie, all feature vectors are the same) and thus poor performance in downstream applications.
\section{Experiments}\label{sec:exp}
%\vspace{-0.3em}
In this section, we empirically evaluate the analysis on our provided answers:
%evaluate the two above mentioned problems: 
(1) Can we go beyond the KL-based objective (\S\ref{sec:exp_obj}): we apply various $f$-MICL objectives to popular vision and language datasets. In particular, we show that under the same network architecture design, $f$-MICL can always provide a better choice of objective. We observe that the best-performing $f$-divergence is largely dataset dependent.
(2) Can we design a better similarity function (\S \ref{sec:similarity}): we show that the proposed $f$-Gaussian similarity is more powerful than the heuristic cosine similarity, regardless of the choice of $f$. 
Moreover, we confirm empirically that $f$-MICL extends the nice property of alignment and uniformity  in \S\ref{sec:au}.

\vspace{-0.2em}
\subsection{Experimental settings}
\vspace{-0.2em}
%Our experiments are conducted on a GPU cluster with \texttt{NVIDIA T4} and \texttt{P100} using \texttt{pytorch}. 
Our detailed settings can be found in Appendix \ref{app:add_exp}. 
In all our experiments, we change only the objective of different methods for fair comparison.
We use the $f$-Gaussian similarity in $f$-MICL by default. 
%Our Algorithm \ref{alg:f-MICL} is trained on a GPU cluster with \texttt{NVIDIA T4} and \texttt{P100}.

\textbf{Vision task.} \,
Our vision datasets include CIFAR-10 \citep{krizhevsky2009learning}, STL-10 \citep{coates2011analysis}, TinyImageNet \citep{chrabaszcz2017downsampled}, and ImageNet \citep{deng2009imagenet} for image classification. 
After learning the feature embeddings, we evaluate the quality of representations using the test accuracy via a linear classifier. \teal{Note that we use $\alpha=40$ across all vision experiments}. \\
(1) Smaller datasets: For feature encoders,  we use ResNet-18 \citep{he2016deep} for CIFAR-10;  ResNet-50 \citep{he2016deep} for the rest. Our implementation is based on  SimCLR \citep{ChenKNH20}, \teal{where we used the same 3-layer projection head during training}. All models are trained for \teal{800} epochs.\\
(2) ImageNet: We choose Vision Transformer (ViT-S) \citep{vit} as our feature encoder. We choose the smaller ViT-S model with 6 blocks because larger ViT models are extremely expensive to train on GPUs. 
Our implementations are based on MoCo V3 \citep{mocov3}, \teal{where models are trained for 1000 epochs}.

%We follow the SimCLR sampling method.
%GZ: no need to mention this as we already did in the alg
%Implementation wise,  % no need to emphasize this
%\blue{Note that for ImageNet experiments, we choose a smaller batch size 256 rather than 4096, and train for 100 epochs rather than 800 because training on GPU cluster is much slower than on pricey cloud TPUs. Thus our baseline results might drop when compared to the original reports.}

\textbf{Language task.} To show the wide applicability of our $f$-MICL framework, we also conduct experiments on a natural language dataset, English Wikipedia \citep{gao2021simcse}. We follow the experimental setting in \citep{gao2021simcse}, which applies BERT-based models to SimCLR \citep{devlin2019bert,liu2019roberta}. Specifically, we choose the BERT$_{\tt base}$ model due to limited computing resources. \teal{For $f$-MICL objectives, we choose $\alpha=409600$.} The application task is called semantic textual similarity (STS \citep{agirre2013sem}) and we report the averaged Spearman's correlation in Table \ref{tab:accs} for comparison.

% over RoBERTa$_{\tt large}$ 

%\textbf{Optimizers and Schedulers.} \, For smaller vision tasks, we use SGD with momentum as our optimizer, and apply the cosine learning rate schedule \citep{loshchilov2016sgdr}. For the language dataset and ImageNet (training ViT), we use Adam \citep{KingmaBa14} with the weight decay \citep{loshchilov2018decoupled}.

%For fair comparison, we always keep the same network architecture and \blue{hyperparameters}, while only changing the objective.

%Our results may be different with those reported in \citep{tsai2021self} and \citep{chen2020big} because we use different model architectures (SimCLR vs SimCLRv2) and training procedures.

\vspace{-0.2em}
\subsection{$f$-MICL objectives}
\label{sec:exp_obj}
\vspace{-0.2em}
% In Table \ref{tab:accs} we  compare our $f$-MICL framework with various popular baselines, including SimCLR \citep{ChenKNH20}, Uniformity \citep{wang2020understanding} and Relative Predictive Coding \citep[RPC,][]{tsai2021self}. 
\textbf{Smaller datasets and language task.} 
We first compare $f$-MICL with several InfoNCE-based  contrastive learning algorithms (i.e., SimCLR \citep{ChenKNH20}, MoCo \citep{he2020momentum}, and AU \citep{wang2020understanding}) on smaller datasets and the language task in Table \ref{tab:accs}. Here we choose four $f$-divergences with the best overall performance. See Appendix \ref{app:add_exp} for results on other $f$-divergences.

From Table \ref{tab:accs} we observe that: (1) As we have shown in  \S\ref{sec:infonce} that $f$-MICL generalizes InfoNCE-based objectives, empirically KL-MICL achieves similar performance to the baselines. In practice, we can tune the hyperparameter $\alpha$ such that KL-MICL outperforms the InfoNCE-based objectives.
(2) KL-MICL is not always the optimal choice. We can see that the best-performing $f$-MICL objectives refer to four different $f$-divergences on four datasets.
%four different $f$-MICL objectives achieves the best performance respectively. \begin{wrapfigure}{r}{0.6\textwidth}

\begin{table*}[t]
\centering
\caption{
%test accuracy (\%) on the vision datasets. For the Wikipedia dataset we evaluate the semantic textual similarity (STS) via the Spearman's correlation. 
We compare the test accuracy (\%) obtained with the linear evaluation on the vision datasets. On the Wikipedia dataset, we compare the semantic textual similarity (STS) via the Spearman's correlation.
%For ImageNet we train for $100$ epochs with batch size $256$ due to computation limit. 
%A drop of performance compared to the baseline methods is because of different experimental settings (i.e.~smaller batch size and less epochs due to resource limit), 
For each dataset and each method we take three different runs to get the mean and the standard derivation.
%{\bf RKL:} Reversed Kullback--Leibler. 
% {\bf Uniformity:} the default setting in \citet{wang2020understanding}. 
% We use $k$-NN and linear evaluation on CIFAR-10 and linear evaluation for other datasets.  
}
\label{tab:accs}
%\small
\setlength\tabcolsep{4pt}
\begin{tabular}{c@{\hskip2ex}cccc@{\hskip2ex}cccc}
\toprule
\multirow{2}{*}[-.6ex]{\bf Dataset} & \multicolumn{3}{c}{\bf Baselines} & \multicolumn{4}{c}{\bf $f$-MICL}\\
\cmidrule(l{0pt}r{10pt}){2-4}\cmidrule(l{-1pt}r{-1pt}){5-8}
&  MoCo & SimCLR %& MoCo
& AU %& RPC  %& RKL & Neynman 
& \phantom{kk}KL\phantom{kk} & JS & Pearson  & VLC \\
\midrule
CIFAR-10 &90.30$_{\pm 0.19}$ & 89.71$_{\pm 0.37}$ %& 88.66 
& 90.41$_{\pm 0.26}$ %& 90.39$_{\pm 0.25}$   % & 10.00 & 10.00 
&\textbf{90.61}$_{\bf\pm 0.47}$ & 89.66$_{\pm 0.28}$
& 89.35 $_{\pm 0.52}$ & 89.13 $_{\pm 0.33}$\\
%  CIFAR-100 & 62.77$_{\pm 0.17}$ & 62.75$_{\pm 0.45}$  %& {\bf 65.73} 
%& 62.51$_{\pm 0.36}$ & 62.66$_{\pm 0.39}$ %& 1.00 & 1.00
%& 63.00 $_{\pm 0.44}$ & {\bf 63.11}$_{\bf \pm 0.33}$ & 61.69$_{\pm 0.57}$  & 61.19$_{\pm 0.29}$ \\
STL-10 &83.69$_{\pm 0.22}$ & 82.97$_{\pm 0.32}$ & 84.44$_{\pm 0.19}$  %&82.41 $_{\pm 0.14}$
&85.33$_{\pm 0.39}$ & {\bf 85.94}$_{\bf \pm 0.17}$ & 82.64$_{\pm 0.37}$ & {\bf 85.94}$_{\bf \pm 0.72}$ \\
TinyImageNet & 35.72$_{\pm 0.17}$ & 30.56$_{\pm 0.28}$ &41.20$_{\pm 0.19}$ %& 34.95$_{\pm 0.25}$
& 39.46$_{\pm 0.20}$ & 42.98$_{ \pm 0.18}$  &\bf 43.45$_{
\bf \pm 0.54}$  & 38.65$_{\pm 0.45}$ \\
%ImageNet & 58.66$_{\pm 0.2}$ & 57.68$_{\pm 0.25}$ & 59.09$_{\pm 0.24}$ & 57.11$_{\pm 0.36}$ & 58.95$_{\pm 0.33}$ & \textbf{61.21$_{\bf \pm 0.29}$} & 55.32$_{\pm 0.58}$  & 54.24$_{\pm 0.62}$\\

\midrule
Wikipedia & 77.88$_{\pm 0.15}$ & 77.40$_{\pm 0.12}$ & 77.95$_{\pm 0.08}$ %& 68.32 $_{\pm 0.23}$ %& -- & -- 
& {\bf 78.02}$_{\bf \pm 0.13}$ & 76.76$_{\pm 0.09}$ & 77.59$_{\pm 0.12}$& 55.07$_{\pm 0.13}$\\
\bottomrule
\end{tabular}
\end{table*}

\begin{table*}[t]
\centering
\caption{
We compare the test accuracy (\%) with SOTA methods on ImageNet. We take three different runs to get the mean, where the standard derivations are less than 0.1\% for $f$-MICL. 
}
\label{tab:imnet}
%\small
\setlength\tabcolsep{6pt}
\begin{tabular}{c@{\hskip2ex}cccccc@{\hskip2ex}ccc}
\toprule
\multirow{2}{*}[-.6ex]{\bf Dataset} & \multicolumn{6}{c}{\bf Baselines} & \multicolumn{3}{c}{\bf $f$-MICL}\\
\cmidrule(l{0pt}r{10pt}){2-7}\cmidrule(l{-1pt}r{-1pt}){8-10}
& SwAV  %& MoCo
& BYOL%& RPC  %& RKL & Neynman 
& Barlow Twins & VICReg & R\'enyiCL & MoCo v3  & KL & JS & Pearson\\
\midrule
ImageNet & 75.3 & 74.3 & 73.2 & 73.2 & \bf 76.2 & 73.2 & 73.9 & \bf 76.5 & 74.6\\
\bottomrule
\end{tabular}
\end{table*}

\begin{table}[t]
\centering
\caption{Comparison between the cosine and \teal{$f$}-Gaussian similarities on CIFAR-10 with the test accuracy (\%). For the Tsallis-$\alpha$ divergence we take $\alpha = 3$. %\blue{SH: Squared Hellinger divergence; Tsa}.
%GZ: no need to mention this
% Across all the $f$-divergences we have tried, Gaussian similarity consistently outperforms cosine similarity.
% KL: Kullback--Leibler; JS: Jensen--Shannon; Pearson: Pearson $\chi^2$; SH: Squared Hellinger; Tsallis: Tsallis-$\alpha$; VLC: Vincze--Le Cam. 
}
%\vspace{-0.5em}
\label{tab:cos_vs_gaussian}
%\small
\setlength\tabcolsep{15pt}
\scalebox{1.0}{\begin{tabular}{cccccccccc}
\toprule
Similarity %& CPC %& MoCo & RPC 
& KL & JS & Pearson & SH & Tsallis-$\alpha$ & VLC \\ \midrule\midrule
\multirow{2}{*}{Cosine} %& 84.77 %& 88.7 & 84.11 
& 89.95 & 88.06 & 87.79 & 87.06 & 88.55 & 10.00  \\
& ${\pm 0.26}$& ${\pm 0.33}$& ${\pm 0.42}$& ${\pm 0.55}$& ${\pm 0.28}$& ${\pm 0.00}$\\
\midrule
\multirow{2}{*}{Gaussian} %& / %& / & / 
& \bf 90.61 & \bf 89.66 & {\bf 89.35} & \bf 89.52 & \bf 89.15 &\bf 89.13\\
& ${\pm 0.47}$& ${\pm 0.28}$& ${\pm 0.52}$& ${\pm 0.25}$& ${\pm 0.42}$& ${\pm 0.33}$\\
\bottomrule
\end{tabular}}
%\vspace{-0.5em}
 \end{table}

The above results indicate that
%except non-satisfying $f$, 
$f$-MICL can provide a wide range of objective choices for the downstream tasks. Although how to derive an optimal $f$-divergence deserves more study in theory, in practice we can select the best $f$ among several common $f$-divergences on a validation set.

%of objectives that could be optimal depending on the task. 
% While the choice of $f$ might require more study in theory, in practice we have to rely on a validation set for quick investigation. 

Besides the $f$-divergences in Table \ref{tbl:choices_f_div},
in Theorem \ref{thm:uniformity} we have identified non-satisfying $f$-divergences.
In our experiments, we found that
applying these $f$-divergences such as the RKL and Neyman $\chi^2$ divergences would result in feature collapse. For example, in Figure \ref{fig:uniformity_batch} we show  that with RKL  the features all collapse to a constant. 

\begin{figure*}
    \centering
    \begin{subfigure}{0.61\textwidth}
        \includegraphics[width=1.0\textwidth]{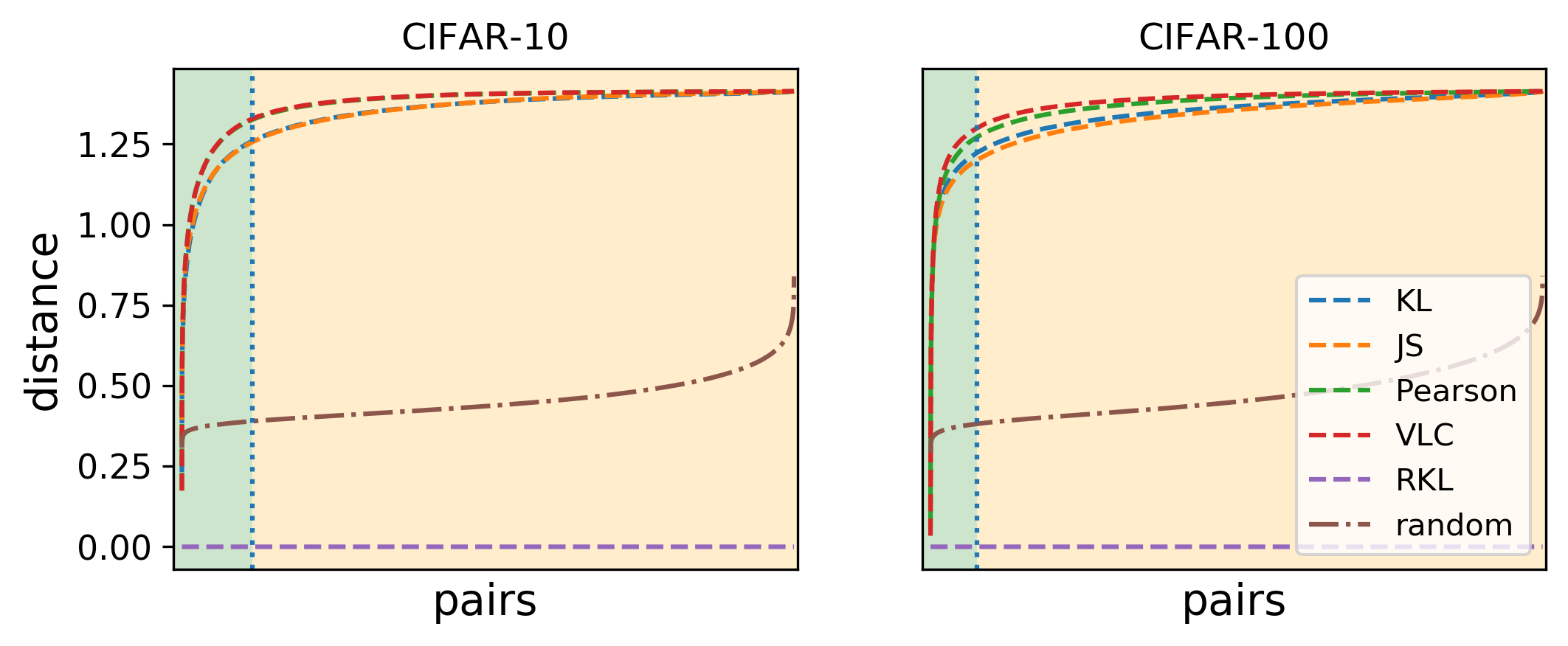}
        %\label{fig:uniformity}
        %\caption{}
    \end{subfigure}
    \begin{subfigure}{0.31\textwidth}
        \vskip0.75ex
        \includegraphics[width=1.0\textwidth]{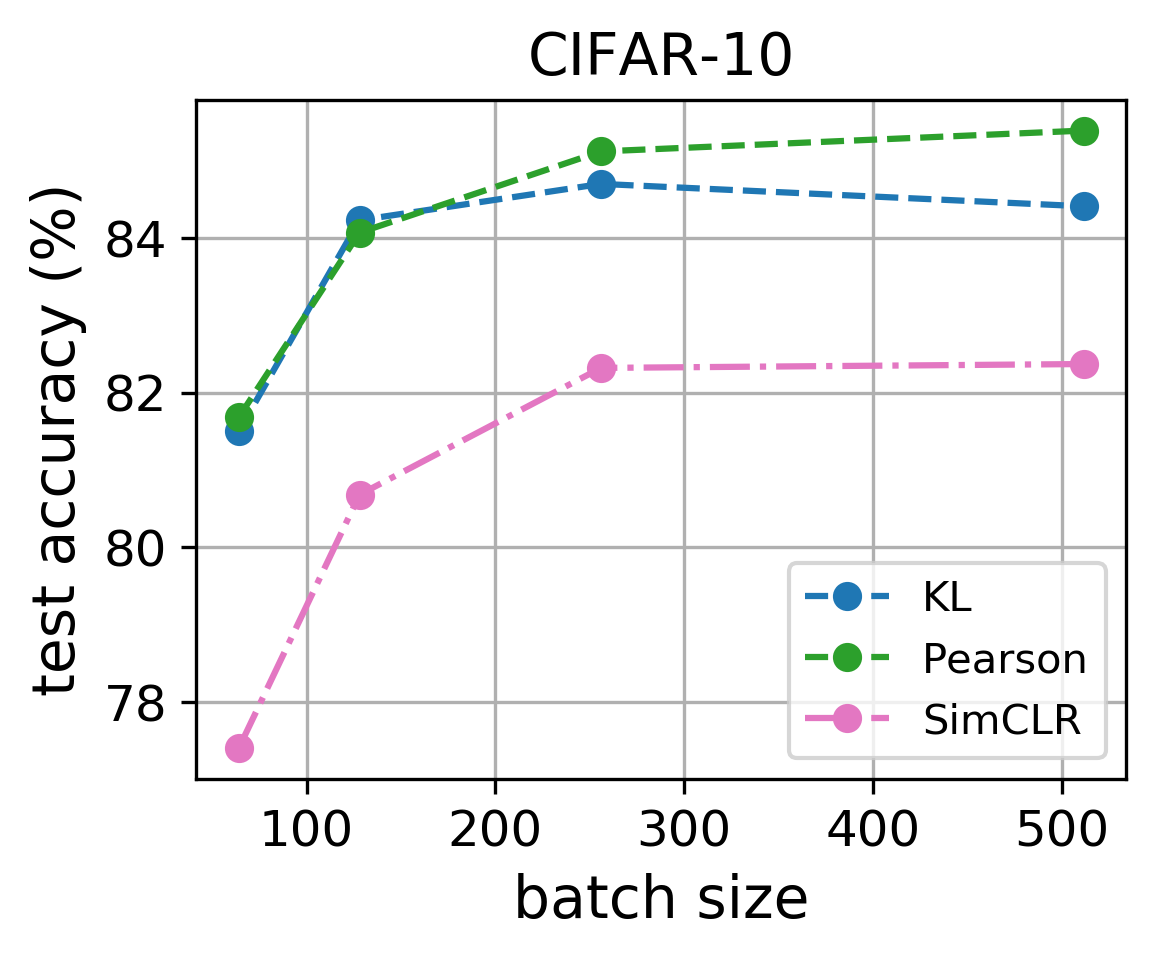}
        %\label{fig:batch_size}
        %\caption{}
    \end{subfigure}
    %\vspace{-1em}
    \caption{ {\bf (left and middle)} Distances between pairs of normalized features within a batch. {\bf Green region:} similar pairs. {\bf Orange region:} dissimilar pairs. $f$-MICL gives nearly uniform distances for dissimilar pairs for the $f$-divergences in Table \ref{tbl:choices_f_div}. For non-satisfying $f$-divergences such as the RKL, the features collapse to a constant and thus the distances are zero. {\bf (right)} The test accuracy v.s.~the batch size after training $200$ epochs for all algorithms.}
    \label{fig:uniformity_batch}
    %\vspace{-1em}
\end{figure*}

\textbf{ImageNet results:} To further demonstrate the efficacy of $f$-MICL we then compare  with several popular self-supervised learning methods,  including both contrastive-based and non contrastive-based ones. 
These methods can be categorized into the ResNet-based (SwAV \citep{swav}, BYOL \citep{huynh2020boosting}, Barlow Twins \citep{barlow}, VICReg \citep{vicreg}) and R\'enyiCL \citep{lee2022renyicl}, and the ViT-based \citep{vit}. For the ResNet-based methods, we directly retrieve results from \citep{vicreg}, which are obtained by training ResNet-50 models for 1000 epochs. Chen et al.~\citep{mocov3} show that these two types of methods are directly comparable in terms of the model size and supervised learning performance.
For the ViT-based method and our $f$-MICL, we apply ViT-S for 1000 epochs. Specifically, our $f$-MICL follows MoCo v3 with only the objective changed for different choices of the $f$-divergence.

Results in Table \ref{tab:imnet} show that: (1) By only changing the objective function, our method improves MoCo v3 by 2.1\% using JS-MICL. (2) $f$-MICL objectives are comparable with the ResNet-based methods (e.g., SwAV and R\'enyiCL). 

Overall, our experiments confirm that $f$-MICL can provide a better choice of objective than InfoNCE on a variety of datasets, tasks, and encoder architectures.

\subsection{$f$-Gaussian similarity}
\label{sec:similarity}
Next, we want to examine the effect of our similarity function \teal{while fixing the $f$-divergences}. 
In Table \ref{tab:cos_vs_gaussian} we compare the cosine  and $f$-Gaussian similarities for different $f$-divergences on CIFAR-10. It can be seen that  under our $f$-MICL framework the $f$-Gaussian similarity consistently outperforms the cosine similarity for various  $f$-divergences\footnote{As we have shown the equivalence between cosine and Gaussian similarity on KL, the difference results on KL just show the choice of different scaling factors.}. This agrees with our Theorem \ref{lem:optimal_sim} and \cref{eq:optimal_similarity}, and  also implies the validity of Assumption \ref{assmp:joint_vMF}. In particular, we identify that without the theoretical guarantee, the heuristic cosine similarity would fail for certain $f$-MICL objectives (\eg, VLC).

\subsection{Alignment and uniformity test} 
%\vspace{-0.3em}
\label{sec:au}
We empirically check the properties of alignment and uniformity for $f$-MICL by
% Finally, to check the uniformity of feature vectors (Theorem \ref{thm:uniformity}) 
plotting the pairwise distance $\|x_i^g - x_j^g\|$ of the feature representations within the same batch on CIFAR-10. We compute the distances between the normalized features of every pair from a random batch, and then sort the pairs in increasing order. From Figure \ref{fig:uniformity_batch} we can see that $f$-MICL gives nearly uniform distances for dissimilar pairs (orange regions) on both datasets with various proper $f$-divergences.  In contrast, a random initialized model gives a less uniform distribution for dissimilar pairs. Besides, we observe small pairwise distances for similar pairs (green regions). 

\subsection{Sensitivity to batch size} 
Finally, we study the sensitivity to the batch size of our $f$-MICL framework on  CIFAR-10.
On the right panel of Figure \ref{fig:uniformity_batch}, we evaluate the classification accuracy by varying the batch size for different $f$-divergences and SimCLR. We can see that for all different batch sizes and with the proper choice of $f$-divergences, our performance is always better than SimCLR. In other words, we require fewer negative samples to achieve the same performance.

%On CIFAR-100 there are less similar pairs compared to CIFAR-10 as there are more classes. 
%Draw  some visualization of latent representations using t-SNE (maybe on ImageNet)
\begin{comment}
\begin{figure*}
    \centering
    \begin{subfigure}{0.61\textwidth}
        \includegraphics[width=1.0\textwidth]{images/uniform.png}
        %\label{fig:uniformity}
        %\caption{}
    \end{subfigure}
    \begin{subfigure}{0.31\textwidth}
        \vskip0.75ex
        \includegraphics[width=1.0\textwidth]{images/batch_size.png}
        %\label{fig:batch_size}
        %\caption{}
    \end{subfigure}
    \vspace{-1em}
    \caption{ {\bf (left and middle)} Distances between pairs of normalized features within a batch. {\bf Green region:} similar pairs. {\bf Orange region:} dissimilar pairs. $f$-MICL gives nearly uniform distances for dissimilar pairs for the $f$-divergences in Table \ref{tbl:choices_f_div}. For non-satisfying $f$-divergences such as the RKL, the features collapse to a constant and thus the distances are zero.}
    %{\bf (right)} The test accuracy v.s.~the batch size after training $200$ epochs for all algorithms. \red{will change the graph soon}}
    \label{fig:uniformity_batch}
    %\vspace{-1em}
\end{figure*}
\end{comment}

%\vspace{-0.3em}
\section{Related Work}
\label{app:related}

\textbf{Contrastive learning.} \; Self-supervised contrastive learning learns representations by contrasting sample pairs. Recently it has been shown analytically that improving the contrastiveness can benefit the downstream applications \citep{SaunshiPAKK19,tosh2021contrastive}. For popular contrastive learning methods such as Contrastive Predictive Coding (CPC) \citep{oord2018representation}, SimCLR \citep{ChenKNH20}, and MoCo \citep{he2020momentum}, their loss functions can be interpreted as a lower bound of mutual information, which is essentially the KL divergence between the joint distribution and the product of margin distributions. Besides the KL divergence, other statistical divergences or distances have been individually studied under the context of contrastive learning, \eg, the Wasserstein distance \citep{ozair2019wasserstein}, Pearson $\chi^2 $ divergence \citep{tsai2021self}, and Jensen--Shannon divergence \citep{hjelm2018learning}. 

% \textbf{Uniformity} Our work is inspired by \citept{wang2020understanding}, which shows the alignment and uniformity in contrastive learning. However, \citept{wang2020understanding}  focuses on the InfoNCE loss and the loss function $\Lc_{\tt uniform}$ they consider is somewhat \emph{ad hoc}. In contrast, our $f$-MICL objective is directly built from the $f$-MI theory. Also, \citept{wang2020understanding} proves uniformity in the limit of infinite samples, while our Theorem \ref{thm:uniformity} only needs finite samples. 
%GZ: this should appear in the uniformity section not here

\textbf{$f$-divergences} have been widely used in generative models \citep{nowozin2016f} and domain adaptation \citep{acuna2021fdomainadversarial}, for measuring the discrepancy of two distributions, where the variational lower bound is often employed for estimation.
Compared to $f$-GAN \citep{nowozin2016f} and $f$-DAL \citep{acuna2021fdomainadversarial} which minimize the $f$-divergence between two different distributions, our $f$-MICL objective is to maximize the $f$-divergence between the joint distribution and the product of marginal distributions. This agrees with our purpose of contrasting sample pairs. Moreover, we provide a theoretical criterion for choosing proper $f$-divergences. 
%Furthermore, while $f$-GAN
% aims to learn a generative model and thus \emph{minimizes} the $f$-divergence, our $f$-MICL intends to encourage the contrastiveness and \emph{maximizes} the $f$-divergence. 
%GZ: avoid using contrastiveness

\textbf{Mutual Information} also plays an important role in the context of deep representation learning \citep{tian2020contrastive,bachman2019learning,hjelm2018learning,TianSPKSI20, poole2019variational, belghazi2018mine}.  \emph{Loss function wise}, our losses partially cover the losses in the literature and generalizes them: e.g., \cite{poole2019variational} considers several variational lower bounds of mutual information, where we generalize the DV objective;
(b) \emph{application-wise}, none of them considers contrastive learning: e.g., \cite{poole2019variational} considers mutual information estimation, \cite{belghazi2018mine} improves adversarial generative models, \cite{hjelm2018learning} considers representation learning that maximizes local and global information. 

\textbf{Metric learning.} Our work is closely related to metric learning \citep{kaya2019deep,suarez2018tutorial}, which aims to learn a distance metric bringing similar objects closer and distancing dissimilar objects further. In contrastive learning, a pre-defined similarity metric, \eg, the cosine similarity \citep{ChenKNH20,he2020momentum} or a bilinear function \citep{oord2018representation,tian2020contrastive,henaff2019data} is commonly used to measure the sample similarity. These pre-designed metrics may not necessarily lead to satisfactory performance in practice. Comparably, the design of our similarity function is empirically tailored for contrastive learning.

\blue{Finally, we summarize existing representation learning methods that utilize $f$-divergences and compare with $f$-MICL in \Cref{tab:fdivergence} for a clear view of the literature.  }

\begin{table}[t]
\centering
\caption{\blue{Comparison between different representation  learning methods that apply $f$-divergences.}}
\label{tab:fdivergence}
%\small
\setlength\tabcolsep{2pt}
\scalebox{1.0}{\begin{tabular}{cccccc}
\toprule
Method
& Objective & $f$-divergence & Similarity %& %Bound 
& Task \\
\midrule\midrule
CPC \citep{oord2018representation} 
& InfoNCE & KL & log-bilinear %& conjugate 
& predictive coding  \\
RPC \citep{tsai2021self} 
& RPC & Pearson $\chi^2$ & cosine % & conjugate 
& predictive coding\\ 
\midrule
MINE \citep{belghazi2018mine} 
& DV bound of MI & KL & neural net % & conjugate 
& GAN generation \\ 
\midrule
DIM \citep{hjelm2018learning} 
& JSD/InfoNCE & JS \& KL & neural net  %& conjugate 
& representation learning \\
\citet{poole2019variational} & MI bounds & KL & joint/separable %& TUBA 
& representation learning\\
\midrule
SimCLR \citep{ChenKNH20} & InfoNCE & KL & cosine %& conjugate 
& contrastive learning \\
MoCo \citep{he2020momentum} & InfoNCE & KL & cosine %& conjugate 
& contrastive learning\\
AU \citep{wang2020understanding} & InfoNCE & KL & Gaussian %& conjugate 
&contrastive learning \\
\bf $f$-MICL (Ours) & $f$-MI &
general & $f$-Gaussian %& conjugate 
& contrastive learning\\

\bottomrule
\end{tabular}}
%\vspace{-0.5em}
 \end{table}
%\vspace{-0.5em}
\section{Conclusion}
%\vspace{-0.3em}

We developed $f$-MICL for contrastive learning, which generalizes the  KL-based mutual information to the $f$-mutual information. With $f$-MICL we provided a broad spectrum of objective choices with better downstream performance. We also proposed a novel $f$-Gaussian similarity function, which shows superior performance to the commonly used cosine similarity. In addition, we confirmed the generalization of $f$-MICL by comparing with popular InfoNCE-based objectives.   Empirically, we exhibited the efficacy of $f$-MICL across a wide range of datasets  from both vision and 
natural language.

\textbf{Limitations and future work.} 
While $f$-MICL provides a variety of objective functions, it is yet unclear how to choose an optimal $f$ based on a task and a dataset in theory, such that we usually rely on a validation set in practice for selection.
An interesting future work is to learn an optimal $f$-divergence using a parametrized neural network. 
Moreover, \cite{lee2022renyicl} applied Skew R\'enyi divergence for contrastive learning. However, we observe that applying R\'enyi-MICL naively leads to a large variance (similar to Section 4.1 in \citealt{lee2022renyicl}), and we leave the discussion on skew divergences for future works. Additionally, \cite{McAllesterS20} showed that there exist some inherent statistical limitations on accurately estimating the mutual information with various lower bounds. In future work it would be interesting to examine if such limitations extend to $f$-MI, and if a limited estimation of $f$-MI necessarily affects $f$-MICL whose goal is to compare and learn representations through (lower bounds of) $f$-MI.
% In this work we proposed $f$-MICL, which generalizes the KL-based objective with the $f$-mutual information in contrastive learning. 
% By making the assumption that the joint feature distribution is proportional to a Gaussian kernel, we design an $f$-Gaussian similarity function that is more intuitive and performs better than the \emph{de facto} cosine similarity. Empirically, we showed the efficacy of $f$-MICL across a wide array of datasets and the advantage of using the $f$-Gaussian similarity.

%a contrastive learning framework with a generalization of mutual information, called \emph{$f$-mutual information} ($f$-MI). Our objective corresponds to the variational lower bound of $f$-divergences. By making the assumption that the joint feature distribution is proportional to a Gaussian kernel, we naturally 
%characterized the properties of the alignment and uniformity. 
%Our $f$-MI based objective can be well estimated with finite samples. Empirically, we showed the efficacy of $f$-MICL across a wide array of datasets and the advantage of using the Gaussian similarity. Our results imply the following: 

%\begin{itemize}
%\item  even though mutual information is widely used in contrastive learning, generalization to $f$-MI can bring us better performance; 
%\item the cosine similarity, the \emph{de facto} option in contrastive learning, can be replaced with the more effective Gaussian kernels. 
%\end{itemize}

%As our work shows interesting connection between contrastive learning and kernel methods, it would be promising to explore more RBF kernels in contrastive learning. 

\newpage
%\section*{Ethics Statement}
%{\color{blue}{}}
% \section*{Reproducibility Statement}
% {\color{blue}{Our novel $f$-MICL objectives are implemented in anonymous downloadable source code as in the supplementary. For theoretical results, clear explanations of any assumptions and complete proofs are included in Appendix \ref{app:theory} and Appendix \ref{app:proofs}. For detailed experimental settings and complete experimental results, the complete descriptions can be found in Appendix \ref{app:add_exp}. }}

\section*{Acknowledgement}
We thank the reviewers and the action editor for their constructive comments. Part of this work was performed during YL's internship at NRC. YY thanks NSERC and CIFAR for funding support. 
Resources used in preparing this research were provided, in part, by the Province of Ontario, the Government of Canada through CIFAR, and companies sponsoring the Vector Institute.

\bibliography{main}
\bibliographystyle{tmlr}

\newpage
\appendix
%\section{Appendix}

\section{Additional theoretical results}\label{app:theory}

In this appendix, we provide additional theoretical results, including additional $f$-divergences and the theory for weighting parameters.

\textbf{Notations.} We assume that a dominating measure $\lambda$ (e.g.~Lebesgue) is given and all other probability measures are represented as some density w.r.t.~$\lambda$. 
%We denote $D_f(p\|q)$ as the $f$-divergence between two density functions $p$ and $q$.
%GZ: I will explain later
Given the joint density $p(x,y)$, we denote $p(x)  = \int p(x,y) \mathrm{d}\lambda(y)$ and $p(y) = \int p(x,y) \mathrm{d}\lambda(x)$ as the marginals. We use $\supp(\cdot)$ to denote the support of a distribution, and $f^*$ to denote the conjugate of function $f$. Every norm presented is Euclidean. We use $x^g := g(x)$ as the shorthand notation for the feature embedding, with $x$ a raw sample. The notation $p_{d}$ stands for the data distribution, and $\squares := p_d\otimes p_d$ means its self product. We denote $p_+$ as the distribution of \emph{positive pairs}, \ie, two samples with similar feature embeddings. The symbol $\circ$ denotes function composition.

\subsection{Additional $f$-divergences}
\label{app:f-div}

We expand Table \ref{tbl:choices_f_div} and give more examples of $f$-divergences in Table \ref{tbl:choices_f_div_app}. As we will see in the proof of Theorem \ref{thm:uniformity}, Table \ref{tbl:choices_f_div} gives a special class of $f$-divergences that guarantees uniformity. A detailed description of $f$-divergences can be found in e.g.~\cite{sason2016f}.

\begin{table*}
\setlength\tabcolsep{8pt}
\caption{A summary of common $f$-divergences. KL: Kullback--Leibler; JS: Jensen--Shannon; SH: Squared Hellinger. For JS, we define $\varphi(u) = - (u+1) \log \frac{1+u}{2}+u \log u$. For Pearson $\chi^2$, we take $f^*(t) = -1$ if $t\leq -2$. For Jeffrey, $\widehat{W} = W + W^{-1}$ and $W(\cdot)$ is the Lambert-$W$ product log function. The Tsallis-$\alpha$ divergence is defined in \protect\cite{tsallis1988possible} and we have $\alpha > 1$ for $f$-divergences. We ignore constant addition $-1/(\alpha - 1)$ because it does not change the optimization problem. The Vincze--Le Cam divergence can be found in \protect\cite{le2012asymptotic} which is closely related to $\chi^2$ and Hellinger divergences. For the Vincze--Le Cam divergence we require $-3 < t < 1$ and $f^*(t) = -1$ if $t\leq -3$. }
\label{tbl:choices_f_div_app}
\begin{center}
	\begin{tabular}{lllll}
	\toprule
	\centering
	{\bf Divergence} & $f(u)$  & $f^*(t)$ & $f'(u)$ & $f^*\circ f'(u)$ \\
	\midrule \midrule 
	KL
	& $u \log u$
	& $\exp(t-1)$
	& $\log u + 1$
	& $u$
	\\
	Reverse KL
	& $- \log u$
	& $-1-\log (-t)$
	& $-1/u$
	& $\log u - 1$
	\\
	JS
	& $\varphi(u)$
	& $- \log(2-e^t)$
	& $\log 2 + \log \frac{u}{1+u}$
	& $-\log 2 + \log( 1 + u)$
	\\
	Pearson $\chi^2$ 
	& $(u-1)^2$ 
	& $t^2/4 + t$
	& $2(u-1)$
	& $u^2 - 1$
	\\
	SH 
	& $(\sqrt{u} - 1)^2$
	& $\frac{t}{1 - t}$
	& $1 - u^{-1/2}$
	& $u^{1/2} - 1$
	\\
	Neyman $\chi^2$ 
	& $\frac{(1 - u)^2}{u}$
	& $2 - 2\sqrt{1-t}$ 
	& $1 - u^{-2}$
	& $2 - 2 u^{-1}$
	\\
	Jeffrey
	& $(u - 1) \log u$
	& $\widehat{W}(e^{1-t}) + t - 2$
	& $1 - u^{-1} + \log u$
	& $\widehat{W}(e^{1/u}/u) + \log u - \frac{1+u}{u}$
	\\
	Tsallis-$\alpha$
	& $\tfrac{u^\alpha}{\alpha - 1}$
	& $(\tfrac{\alpha - 1}{\alpha}t)^{\tfrac{\alpha}{\alpha-1}}$
	& $\frac{\alpha}{\alpha - 1} u^{\alpha - 1}$
	& $u^\alpha$
	\\
	Vincze--Le Cam
	& $\frac{(u-1)^2}{u + 1}$ 
	& $4 - t - 4\sqrt{1-t}$
	& $\frac{(u-1)(u+3)}{(u+1)^2}$
	& $3 - \frac{4}{u+1}$\\
	\bottomrule
	\end{tabular}
\end{center}
\end{table*}

\subsection{Weighting parameters}\label{sec:weighting}

In Algorithm \ref{alg:f-MICL} we added a weighting parameter $\alpha$ to balance the alignment and uniformity. We prove that even after adding this parameter we are still maximizing the $f$-mutual information, although with respect to a different $f$. 

\begin{restatable}[\textbf{weighting parameter}]{prop}{weight}\label{prop:scaling}
Given $\alpha > 0$ and a closed convex function $f: \Rb_+ \to \Rb$ such that $f(1) = 0$, define $f_\alpha : \alpha \,\dom f \to \Rb$ with $$f_\alpha(x) = \alpha f\left(\frac{x}{\alpha}\right) - \alpha f\left(\frac{1}{\alpha}\right)$$ for any $x\in \dom f$. Then $I_{f_\alpha}$ is still a valid $f$-mutual information (see Definition \ref{def:fMI}). Besides, by replacing $f$ with $f_\alpha$ in \cref{eq:objective_gaussian} we have the following optimization problem:
$$\sup_{g\in \Gc} \mathbb{E}_{(x, y) \sim \ppos}\left[f'\left(\frac{G_\sigma(\|x^g - y^g\|^2)}{\alpha}\right)\right] - \alpha \mathbb{E}_{(x, y)\sim \squares}\left[f^* \circ f'\left(\frac{G_\sigma(\|x^g - y^g\|^2)}{\alpha}\right)\right],
$$
where $G_\sigma(\|x^g - y^g\|^2) = \mu \exp\left(-\frac{\|x^g - y^g\|^2}{2\sigma^2}\right)$ is the Gaussian kernel. 
\end{restatable}

Note that $\alpha \,\dom f$ means the scalar multiplication of a set which is applied element-wisely. According to Definition \ref{def:fMI}, $f_\alpha$ is also a valid $f$-divergence. This proposition tells us that rescaling the second term with factor $\alpha$ is equivalent to changing the function $f$ to another convex function $f_\alpha$. The transformation from $f$ to $\alpha f\left(\frac{x}{\alpha}\right)$ is also known as right scalar multiplication \citep{urruty1993convex}. Let us now move on to our proof:

\begin{proof}
By definition, we know that $f_\alpha$ is convex and closed with $f_\alpha(1) = 0$, and thus $I_{f_\alpha}$ is a valid $f$-mutual information according to Definition \ref{def:fMI}. Moreover, we have $f_\alpha'(x) = f'(\frac{x}{\alpha})$ for any $x\in \alpha \, \dom f$ and
\begin{align}
f_\alpha^*(t) &= \sup_{x \in \dom f_\alpha} x t - f_\alpha(x) \nonumber \\
&= \sup_{x \in \alpha \dom f} x t - \alpha f\left(\frac{x}{\alpha}\right) + \alpha f\left(\frac{1}{\alpha}\right) \nonumber \\
&= \sup_{\frac{x}{\alpha} \in \dom f} \frac{x}{\alpha}\cdot (\alpha t) - \alpha f\left(\frac{x}{\alpha}\right) + \alpha f\left(\frac{1}{\alpha}\right) \nonumber \\
&= \alpha \sup_{\frac{x}{\alpha} \in \dom f} \left(\frac{x}{\alpha}\cdot t -  f\left(\frac{x}{\alpha}\right)\right) + \alpha f\left(\frac{1}{\alpha}\right) \nonumber \\
& = \alpha f^*(t) + \alpha f\left(\frac{1}{\alpha}\right),
\end{align}
where in the last line we used the definition of $f^*(t)$. Plugging $f_\alpha'$ and $f_\alpha^*$ into \cref{eq:objective_gaussian} yields the desired result. 
\end{proof}

\section{Proofs}\label{app:proofs}

%\numberwithin{thm}{section} % important bit

%% this is well-known and classic; a citation suffices. don't reinvent wheels. 

% \begin{prop}\label{pro:fmi}
% The $f$-mutual information satisfies non-negativity $I_f(X; Y) \geq 0$ and symmetry $I_f(Y; X) = I_f(X; Y)$. Moreover, if $f$ is strictly convex and $f(1) = 0$, then $I_f(X; Y) = 0 \Longleftrightarrow$ X and Y are independent. 
% \end{prop}
% \begin{proof}
% From Jensen's inequality, we always have $I_f(X; Y) \geq 0.$ The second part can be found in e.g.~\cite{esposito2020robust}.   
% \end{proof}

\optimalsim*

\begin{proof}
From Definition \ref{def:fMI}, we are computing the following supremum:
\begin{align}
\sup_{g, s} \int \left( \frac{p_g(x^g, y^g)}{p_g(x^g)p_g(y^g)}s(x^g, y^g) - f^* \circ s(x^g, y^g)\right) d\pmargin^g \otimes \pmargin^g.
\end{align}
Suppose $s$ is unconstrained and we fix $g$. The optimal solution should satisfy:
\begin{align}
\frac{p_g(x^g, y^g)}{p_g(x^g)p_g(y^g)} \in (\partial f^*)(s_\star(x^g, y^g)),
\end{align}
almost surely for $(x, y) \sim \pmargin \otimes \pmargin$. From (3.11) of \cite{rockafellar1966characterization} this is equivalent to:
\begin{align}
s_\star(x^g, y^g) \in \partial f \left( \frac{p_g(x^g, y^g)}{p_g(x^g)p_g(y^g)}\right).
\end{align}If $f$ is differentiable, then for any $u\in \dom f$, $\partial f(u) = \{f'(u)\}$ is a singleton.
\end{proof}

%\Uniform*

%\Gaussian*

\iffalse

\begin{proof}
Simply combine Proposition \ref{prop:uniform} with Lemma \ref{lem:optimal_sim}. 
\end{proof}

\fi
%\paragraph{Proof of Proposition \ref{prop:scaling}}

%\weight*

\uniformDis*

Note that we say the samples are  ``distributed uniformly'' if the feature vectors form a regular simplex (see Figure \ref{fig:tetrahedron}), and thus the distances between all sample pairs are the same.

\begin{proof}
%Instead of proving the results for the $f$-divergences in Table \ref{tbl:choices_f_div} one by one, in this proof, we provide a systematic way of defining ``suitable'' $f$-divergences. Our proof relies on the following result:
%\setcounter{thm}{9}
%\begin{lem}[{\citealt[][Theorem 2.4.1]{borodachov2019discrete}}]\label{thm:pairwise_config}
%Let $h: (0, 4] \to \Rb$ be a convex and decreasing function defined at $t = 0$ by $\lim_{t\to 0^+}h(t)$ and let $2\leq N \leq d + 1$. Then the vertices of regular $(N - 1)$-simplices inscribed in $\Sb^{d-1}$ with centers at the origin minimize the energy on the sphere $\Sb^{d-1}$, $d\geq 2$, with respect to the kernel $K(x, y) = h(\|x - y\|^2)$. If in addition $h$ is strictly convex, then these are the only minimizing configurations of the optimization problem:
%\begin{align}\label{eq:minimal_config} 
%\min_{x_1, \dots, x_N} \sum_{i\neq j} h(\|x_i - x_j\|^2).
%\end{align}
%\end{lem}
%
%\begin{proof}
From the definition of $h$ it is clear that $h$ is decreasing since $f^*$ and $f'$ are both monotonically increasing white $G_\sigma$ is decreasing. Using $h$ we rewrite the second term of \cref{eq:objective_sample} as 
\begin{align}
\min_{x_1^g, \ldots, x_N^g \in \Sb^{d-1}} ~ \sum_{i, j} h(\|x_i^g - x_j^g \|^2).
\end{align}
When $N \in [2, d+1]$, there exists a neat characterization of the minimizers, see e.g. \cite{borodachov2019discrete}. We include the proof below for completeness. 

Apply Jensen's inequality, we have:
\begin{align}
\frac{1}{N^2}\sum_{i, j} h(\|x_i - x_j\|^2) &\geq h\left( \frac{1}{N^2}\sum_{i, j} \|x_i - x_j\|^2\right) \nonumber \\
& = h\left(\frac{1}{N^2}\sum_{i, j} \|x_i - x_j\|^2\right)\nonumber \\
&= h\left(\frac{1}{N^2}\sum_{i,j} (2 - 2x_i \cdot x_j)\right) \nonumber \\
& = h\left(2\left(1 -\left\|\frac1N\sum_{i=1}^N x_i\right\|^2\right)\right) \nonumber \\
& \geq h\left(2\right),
\end{align}
where in the first line we used Jensen's inequality; in the third line we used $\|x_i\| = \|x_j\| = 1$ for any $i, j\in [N]$; in the last line we note that $\|\sum_{i=1}^N x_i\| \geq 0$ and $h$ is a decreasing function. When $h$ is strictly convex and decreasing, it is in fact strictly decreasing, and hence the two inequalities above can be attained iff 
\begin{align}
\bar x := \frac1N\sum_i x_i = {\bf 0}, ~~\mbox{ and } \|x_i - x_j\|^2 \equiv c \mbox{ for all } i \ne j,
\end{align}
namely that $\{x_1, \ldots, x_N\}$ form a regular simplex with its center at the origin. We remark that when $h$ is merely convex, points forming a centered regular simplex may form a strict subset of the minimizers.

To see the necessity of $N \leq d+1$, let us note that 
\begin{align}
x_i^\top x_j = \begin{cases} 
1, & i = j \\
-\frac{1}{N-1} , & i\ne j
\end{cases}
,
\end{align}
since 
\begin{align}
\begin{split}
    &\sum_{ij} \|x_i-x_j\|^2 = 2N^2 = N(N-1) c \implies \\
    &c = \frac{2N}{N-1} = 2 + \frac{2}{N-1}.
\end{split}
\end{align}
Performing simple Gaussian elimination we note that the matrix $X^\top X$ has rank $N-1$ where $X = [x_1, \ldots, x_N] \in \R^{d\times N}$. Therefore, we must have $N-1 \leq d$.

% If $(x_1, \dots, x_N)$ form a regular simplex, i.e., $\|x_i - x_j\| \equiv c$ for all $i\ne j$, then the first inequality is tight. Since a regular simplex is symmetric, we also have $\tfrac1N\sum_{i=1}^N x_i = {\bf 0}$, and hence the last inequality is also tight. The proof of the first part is now complete.

% If $h$ is additionally strictly convex, then the first inequality is tight iff $\|x_i - x_j\|$ is a constant for all pairs $(x_i, x_j)$, i.e., $(x_1, \cdots, x_N)$ form a regular simplex. Since $h$ is decreasing, we can easily prove that $h$ is also strictly decreasing given strict convexity. Therefore, the last inequality is tight iff $\sum_{i=1}^N x_i = {\bf 0}$, i.e., the center of the regular simplex is the origin. 

% Since a regular $N$-simplex can be built from a regular $(N-1)$-simplex, we can prove by induction that to embed a regular $N$-simplex in $d$-dimension Euclidean space, we need $N \leq d + 1$. 
%\end{proof}

%If we compare \cref{eq:minimal_config} with \cref{eq:objective_sample}, we obtain that $h(t) = f^* \circ f'(C\exp(-\frac{t}{2\sigma^2}))$ for any $t \in [0, 4]$. Since $f^*$ and $f'$ are monotonic increasing, $h$ is decreasing. 

Lastly, we need to show when $h$ is a (strictly) convex function, which may not always be true depending on the $f$-divergences. We give the following characterization (we ignore the constants $\mu$ and $2\sigma^2$ in Assumption \ref{assmp:joint_vMF} as they do not affect convexity):
\begin{itemize}
\item $h$ strictly convex: $h_{\rm KL}(t) = e^{-t}$, $h_{\rm JS}(t) = \log(1 + e^{-t}) - \log 2$, $h_{\rm Pearson}(t) = e^{-2t} - 1$, $h_{\rm SH}(t) = e^{-t/2} - 1$, $h_{\rm Tsallis}(t) = e^{-\alpha t}$, $h_{\rm VLC} = 3 - \frac{4}{1 + e^{-t}}$;
\item $h$ convex but not strictly convex: $h_{\rm RKL}(t) = -t - 1$ (RKL stands for Reversed Kullback--Leibler, see Appendix \ref{app:f-div});
\item $h$ concave: $h_{\rm Neyman}(t) = 2 - 2 e^t$ (Neyman stands for Neyman $\chi^2$, see Appendix \ref{app:f-div}). 
\end{itemize}
%Note that in the first case, $h$ is also strictly decreasing, which includes all $f$-divergences in Table \ref{tbl:choices_f_div}. Therefore, we have the desired results. 
Only for the last case we do not have the guarantee that the minimizing configurations could form a regular simplex. For RKL, in fact, any configuration that centers at the origin suffices since $h$ is a linear function.  
\end{proof}

\section{Estimation of $f$-MICL Objective}
%\vspace{-0.5em}
\begin{figure}[t]%[13]{r}{0.26\textwidth}
%\vspace{-8mm}
  \begin{center}
    \includegraphics[width=0.23\textwidth]{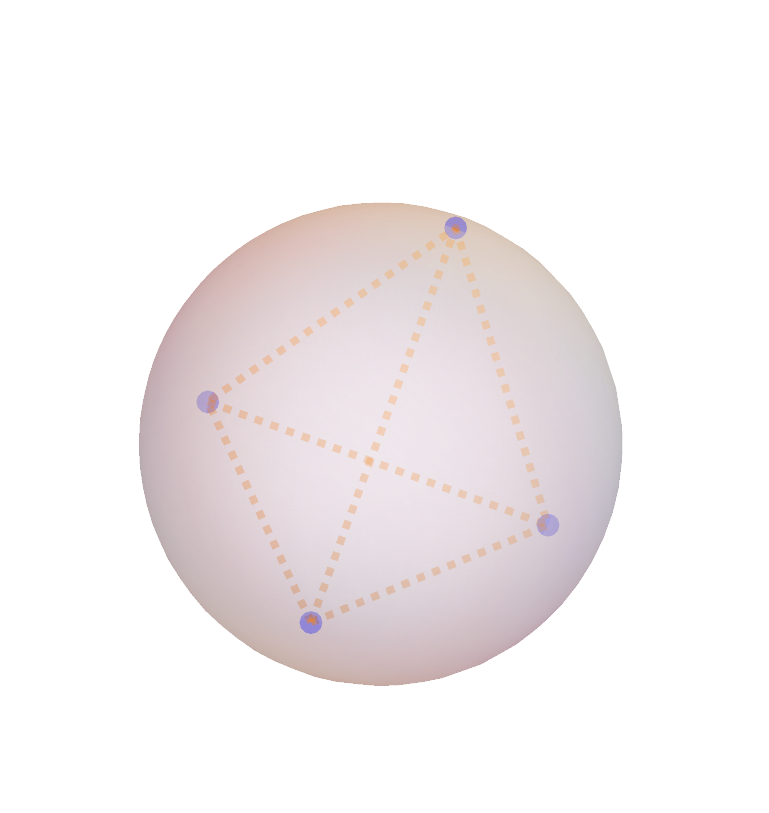}
  \end{center}
  %\vspace*{-7mm}
  %\vspace{-2em}
  \caption{A regular simplex on a hypersphere.}
  \label{fig:tetrahedron}
  %\vspace{-1em}
\end{figure}
In \S~\ref{subsec:imp} we have provided the empirical estimation of our objective in \cref{eq:objective_sample}. However, it remains a question whether our estimation of $f$-mutual information is consistent. 
In this part we derive an upper bound for the estimation error from statistical learning theory.

We denote our $f$-MICL objective as $i_f(X;Y)$ (\cref{eq:objective_gaussian}) and its empirical estimation as $\widehat{i}_f(X;Y)$ (\cref{eq:objective_sample}). After fixing the similarity function , we have $T(x, y) = s_f(x^g, y^g) = f'\circ G_\sigma(\|x^g - y^g\|^2).$
% $$
% i_f(X;Y) = \mathbb{E}_{(x, y) \sim \ppos}[T(x,y)]-\mathbb{E}_{(x, y)\sim \squares}[f^*(T(x,y))],
% $$
% which depends on $T(x, y) = f'\circ G_\sigma(\|x^g - y^g\|^2)$. Recall our empirical estimation of $i_f(X;Y)$ as:
% $$
% \widehat{i_f}(X;Y) = \frac{1}{N}\sum_{i=1}^N [T(x_i,y_i)] - \frac{1}{N(N-1)}\sum_{i\neq j} [f^*\circ T(x_i,x_j)].
% $$
% Given our similarity function $k$ as discussed in the last subsection. We study the estimation error of our objective:
% \begin{align}\label{eq:i_f_true}
% i_f(X;Y) = \mathbb{E}_{(x, y) \sim \ppos}[k(g(x),g(y))]-\mathbb{E}_{(x, y)\sim \squares}[f^*(k(g(x),g(y)))],
% \end{align}
% with our empirical sampling method:
% \begin{align}\label{eq:i_f_estimation}
% \hat{i}_f(X;Y) = \frac{1}{N}\sum_{i=1}^N [k(g(x_i), g(y_i))] - \frac{1}{N(N-1)}\sum_{i\neq j} [f^*\circ k(g(x_i), g(x_j))].
% \end{align}

\begin{restatable}[\textbf{estimation error}]{thm}{Estimation}\label{thm:simclr_moco_negative}
Suppose that the function $T$ is taken from a function class $\Tc$ and define
$\Tc_x$ as the function class of $T(x, \cdot)$ given some $x\in \supp(\pmargin)$. Denote $\Rf_N^P$ to be the Rademacher complexity w.r.t. the distribution $P$ with $N$ i.i.d.~drawn samples.
Then for any $T\in \Tc$, the estimation error $|i_f(X; Y) - \widehat{i}_f(X; Y)|$ is upper bounded with probability at least $1 - \delta$:
\begin{align}\label{eq:esterror}
\textstyle
&2 \Rf_N^{\ppos }( \Tc) + 2 \mu \Bigg( \underset{x\sim \pmargin }{\Eb} \Rf_N^{\pmargin} ( \Tc_x) + \frac{1}{N} \sum_{i=1}^N \Rf_{N - 1}^{\pmargin}( \Tc_{x_j}) \Bigg) %\nonumber
%\\
%&\quad \quad \quad \quad 
+ (r_T + 2 r_f)  \sqrt{\frac{\log 6/\delta}{2(N - 1)}},
\end{align}
with the constants
$
r_T  = f'(\mu) -  f'(\mu e^{-2/\sigma^2})$ and $$r_f = f^*  \circ  f'(\mu)  -  f^*  \circ  f'(\mu e^{-2/\sigma^2}).
$$
% where $r_T$ and $r_f$ are the ranges of $\Tc$ and $f^* \circ \Tc$, respectively. (\red{S: range of the class? does not make sense. You mean the range of the functions in the class.})
\end{restatable}

Here the constant $\mu$ is from our Gaussian kernel in Assumption \ref{assmp:joint_vMF}. Rademacher complexity evaluates the richness of a class of real-valued functions regarding a probability distribution, and its formal definition  can be found in \cite{koltchinskii2001rademacher}.
%(see also Definition \ref{def:rademacher} in Appendix \ref{app:proofs}).

\textbf{Non-i.i.d.~proof.} Our conclusion is theoretically non-trivial since our sample pairs are \emph{non-i.i.d.}: although the individual samples are assumed to be i.i.d., the negative pairs are not independently drawn (\eg, $(x_1, x_2)$ and $(x_1, x_3)$), which makes the derivation challenging. 

Note that the function class $\Tc$ depends on the class of the feature encoder $g$ and the $f$-divergence. Our estimation error eq.~(\ref{eq:esterror}) is composed of three parts:
\begin{itemize}
\item the Rademacher complexity of the function class $\Tc$. In general, if $\Tc$ is richer then its Rademacher complexity is also larger.
\item  the expected Rademacher complexity of the one-side function class $\Tc_x$ and its empirical estimation; 
\item an error term that decreases with more samples.
\end{itemize}
Since the encoders are usually built with neural networks,
%the feature embeddings are usually  neural networks, 
we can use the existing theory \citep{bartlett2019nearly} to give more detailed bounds for the Rademacher complexities of $\Tc$. Specifically, if the Vapnik--Chervonenkis (VC) dimension of $\Tc$ is finite, then our estimation error in eq.~(\ref{eq:esterror}) goes to zero as $N\to \infty$ \citep{mohri2018foundations}.
%(\red{S: what does it mean? Also, here we consider general $T$.}). 

\textbf{Approximation and estimation tradeoff.}\; In order to minimize the estimation error in \cref{eq:esterror}, we should choose a simpler function class $\Tc$ to reduce the Rademacher complexities. However, $\Tc$ should also be rich enough so that \cref{eq:condition_T} can be satisfied, since our objective $i_f(X;Y)$ should approximate the $f$-mutual information $I_f(X;Y)$
if we choose the optimal $T$. Therefore, there is a natural tradeoff between approximation and estimation errors when we change the complexity of $\Tc$.

% Our Theorem \ref{thm:simclr_moco_negative} is non-trivial in the sense that \emph{non-i.i.d.}~negative samples pairs are drawn to estimate the uniformity term, although the samples are i.i.d.~drawn. For instance, $(x_1, x_2)$ and $(x_1, x_3)$ are both negative pairs but they are not independently drawn. 
\section{Additional experimental results}\label{app:add_exp}

We present additional experiment details in this appendix, to further support our experiments in the main paper.  

\subsection{Implementation details}
\label{sec:imp}
In this paper, we follow the implementations in SimCLR (\url{https://github.com/sthalles/SimCLR}) and MoCo v3 (\url{https://github.com/facebookresearch/moco-v3}). For vision tasks, we use ResNet \citep{he2016deep} and ViT-S \citep{vit} as the feature encoder, and we adopt the similar procedure of SimCLR/MoCo for sampling. 
%We find that symmetric loss is not beneficial for downstream tasks performance, thus we use only asymmetric loss for our experiments. 
For the language dataset, we follow the exact experimental setting of \cite{gao2021simcse} and only change the objective. Our experimental settings are detailed below:

\begin{itemize}
    \item Hardware and package: We train on a GPU cluster with \texttt{NVIDIA T4} and \texttt{P100}. The platform we use is  \texttt{pytorch}. Specifically, the pairwise summation can be easily implemented using $\mathtt{torch.nn.functional.pdist}$ from \texttt{pytorch}.
    \item Datasets: the datasets we consider include CIFAR-10, %CIFAR-100 \citep{krizhevsky2009learning}, 
     STL-10 \citep{coates2011analysis}, TinyImageNet \citep{chrabaszcz2017downsampled}, ImageNet \citep{deng2009imagenet} and English Wikipedia \citep{gao2021simcse}. 
    \item Augmentation method: For each sample in a dataset we create a sample pair, a.k.a.~positive pair, using two different augmentation functions. For image samples, we choose the augmentation functions to be the standard ones in contrastive learning, e.g., in \cite{ChenKNH20} and \cite{he2020momentum}. The augmentation is a composition of random flipping, cropping, color jittering and gray scaling. For text samples,  following the augmentation method of \cite{gao2021simcse} we use dropout masks. 
    \item Neural architecture: For CIFAR-10   we use ResNet-18 \citep{he2016deep};  for STL-10, TinyImageNet  we use ResNet-50 \citep{he2016deep}; for ImageNet we use ViT-S \citep{vit}; for the Wikipedia dataset we use BERT$_{\tt base}$ \citep{devlin2019bert}. 
    \item Batch size and embedding dimension: for experiments in CIFAR-10 we choose batch size 512; for STL-10 we choose batch size 64 to accommodate one GPU training; for TinyImageNet, we choose batch size 256; for ImageNet, we choose batch size 1024. For all the vision datasets, we choose the embedding dimension to be 512. Regarding the language dataset, the batch size is 64 with the feature dimension 768. In all of these cases, our assumption $N\leq d + 1$ in Theorem \ref{thm:uniformity} is satisfied.
    \item Hyperparameters: in all our experiments we fix the constant factor $\mu = 1$. We find that in practice the weight parameter $\alpha$ often needs to be large (\eg, in the Wikipedia dataset), which requires moderate tuning. %Note that we also implement RPC \citep{tsai2021self} in our paper. For all the datasets, we follow \citep{tsai2021self} and choose the relative parameters $\alpha=1.0, \beta=0.005$ and $\gamma =1.0$ for all datasets.
    \item Optimizer and learning rate scheduler: For smaller vision tasks, we use SGD with momentum for optimization and the cosine learning rate scheduler \citep{loshchilov2016sgdr}. For the ImageNet task and natural language task, we use Adam with weight decay \citep{loshchilov2018decoupled} and the linear decay scheduler.  
    \item Evaluation metric: for vision tasks, we use $k$-nearest-neighbor ($k$-NN) (only small datasets) and linear evaluation to evaluate the performance, based on the learned embeddings. For the NLP task, we use the Spearman's correlation to evaluate the averaged semantic textual similarity score \citep{gao2021simcse}. 
    \item Baseline methods: for the four baseline methods, we follow the implementations in:
    \begin{itemize}
        \item MoCo: \url{https://github.com/facebookresearch/moco};
        \item SimCLR: \url{https://github.com/sthalles/SimCLR};
        \item Uniformity: \url{https://github.com/SsnL/align_uniform};
        %\item RPC: \url{https://github.com/martinmamql/relative_predictive_coding};
        \item MoCo v3:
        \url{https://github.com/facebookresearch/moco-v3}
 \end{itemize}
    For fair comparison we use the experimental settings in Table \ref{tab:setup} for all the baseline methods, which might differ from the original settings. 
\end{itemize}

Table \ref{tab:setup} gives common choices of hyperparameters for different datasets. Note that we may need to further finetune $\alpha$ and $\sigma$ for different $f$-divergences. See our supplementary code for more details.

\begin{table*}[t]
\setlength\tabcolsep{10pt}
\centering
\caption{Detailed experimental settings. {\tt arch}: the neural network architecture used. $N$: batch size; $d$: the dimension of the feature representation; {\tt lr}: learning rate; $\mu$: the constant factor in $\mu$; $1/(2\sigma^2)$ and $\alpha$ follow from Algorithm \ref{alg:f-MICL}; {\tt epoch}: the number of epochs we run; $k$: the number of nearest neighbors in $k$-NN evaluation. }
\label{tab:setup}
\begin{tabular}{cc@{\hskip1.5ex}cccccccc}
\toprule
 %& CPC %& MoCo & RPC 
Dataset & {\tt arch} & $N$ & $d$ & {\tt lr} & $\mu$ & $(2\sigma^2)^{-1}$ & $\alpha$ & {\tt epoch} & $k$\\ \midrule\midrule
CIFAR-10 & ResNet-18 & 512 & 512 & 0.1 & 1 & 1 & 40 & 800 &200  \\
STL-10 & ResNet-50 & 64 & 512 & 0.1 & 1 & 1 & 40 & 800 & 200 \\
TinyImageNet & ResNet-50 & 256 & 512 & 0.1 & 1 & 1 & 40 & 800 & 200 \\
ImageNet & ViT-S & 1024 & 512 & 0.1 & 1 & 1 & 40 & 1000 & {\tt n/a} \\
Wikipedia & BERT$_{\tt base}$ & 64 & 768 & {\tt 3e-5} & 1 & 20 & 409600 & 1 & {\tt n/a}\\
\bottomrule
\end{tabular}
\end{table*}

%\begin{table}[ht]
%\centering
%\caption{Comparison between InfoNCE and other %$f$-divergences objectives under MoCo architecture on the test accuracies (\%) using $k$-NN evaluation.}
%\caption{Experiments on CIFAR-10 (\red separate to two tables: (1)comparison between Cosine and Gaussian (Only show NWJ,JS,Chi,H2 is enough) (2) compare our results with Moco and SimCLR, should probabaly merge it when we have results for other datasets ) }
%\label{tab:Moco}
%\small
%\setlength\tabcolsep{2pt}
%\begin{tabular}{cccccccccc}
%\toprule
%Objective &%& CPC %& MoCo & RPC 
%MoCo & KL & JS & RKL & Pearson & SH \\ \midrule
 %& 84.77 %& 88.7 & 84.11 
%Accuracy & 88.70 & \textbf{89.98} & 87.25 & 67.39 & 88.17 & 88.48  \\
%\bottomrule
%\end{tabular}
%\end{table}

\begin{table}[b]
\centering
\caption{Ablation study on weighting parameter $\alpha$ for KL and JS divergences on CIFAR-10. We compare test accuracies (\%) for different choices of $\alpha$ using $k$-NN evaluation.}
\label{tab:alpha1}
%\small
\setlength\tabcolsep{16pt}
\begin{tabular}{cccccccc}
\toprule
$\alpha$ & 0.1 & 1 & 10 & 20 & 30 & 40 & 50\\ \midrule
KL &  13.16 & 77.60 & 83.53 & 83.77 & 81.39 & \textbf{84.19} &  82.77  \\
JS & 8.84 & 73.31 & 81.39 & 83.21 & 83.49 & \textbf{84.06} & 82.61 \\
\bottomrule
\end{tabular}
\end{table}

\begin{table}[bt]
\centering
\caption{Ablation study on weighting parameter $\alpha$ for KL and Pearson $\chi^2$ divergences on Wikipedia. We compare the semantic textual similarity (STS) via the Spearman's correlation for different choices of weighting parameter $\alpha$.}
\label{tab:alpha2}
%\small
\setlength\tabcolsep{12pt}
\begin{tabular}{ccccccccc}
\toprule
$\alpha$ & 1 & 10 & $10^2$ & $10^3$ & $10^4$ & $10^5$ & 409600 & $10^6$\\ \midrule
KL & 67.52 & 70.47 & 72.43 & 75.12 & 76.90 & 77.78 & \textbf{78.02} & 77.78   \\
Pearson & 64.58 & 67.78 & 71.58 & 74.03 & 74.95 & 74.40 & \textbf{77.59} & 76.47\\
\bottomrule
\end{tabular}
\end{table}

\subsection{Additional ablation study on weighting parameter}

We provide additional ablation study on the weighting parameter $\alpha$. We perform experiments using  a vision dataset (CIFAR-10) and a language dataset (Wikipedia). For CIFAR-10, we vary $\alpha$ from 0.1 to 50 for KL and JS divergences and run for 200 epochs. \teal{We perform the same experiments on KL when $\alpha=1$ for 800 epochs and observed an accuracy of 83.58\% (lower than SimCLR). This observation further provides empirical evidence that KL-MICL is  different from InfoNCE and needs special tuning on $\alpha$ to perform well.}

Table \ref{tab:alpha1} justifies our choice of $\alpha$ in Table \ref{tab:setup}, where the downstream test accuracy indicates the optimal performance when choosing $\alpha=40$. For the Wikipedia dataset, we observe that a much bigger $\alpha$ is desirable for maximum performance.  We vary $\alpha$ from 1 to $10^6$ for KL and Pearson $\chi^2$ divergences and run for 1 epoch, as there is a large number of samples ($10^6$) in the language dataset. Table \ref{tab:alpha2} justifies our choice of $\alpha$ in Table \ref{tab:setup}, where the best performance is reached at $\alpha=409600$. Such an $\alpha$ is found by starting from $\alpha=100$ and doubling iteratively.

\subsection{Additional experiments}\label{sec:add_exp}

Our final experiments show that $f$-MICL is stable in terms of training and the variation of performance is well controlled. 

\paragraph{Training stability} We depict the training loss curves of different divergences on CIFAR-10 in Figure \ref{fig:curve}. This figure shows that our methods exhibit stable training dynamics with fast convergence.

\begin{figure}
\centering
\includegraphics[width=0.4\textwidth]{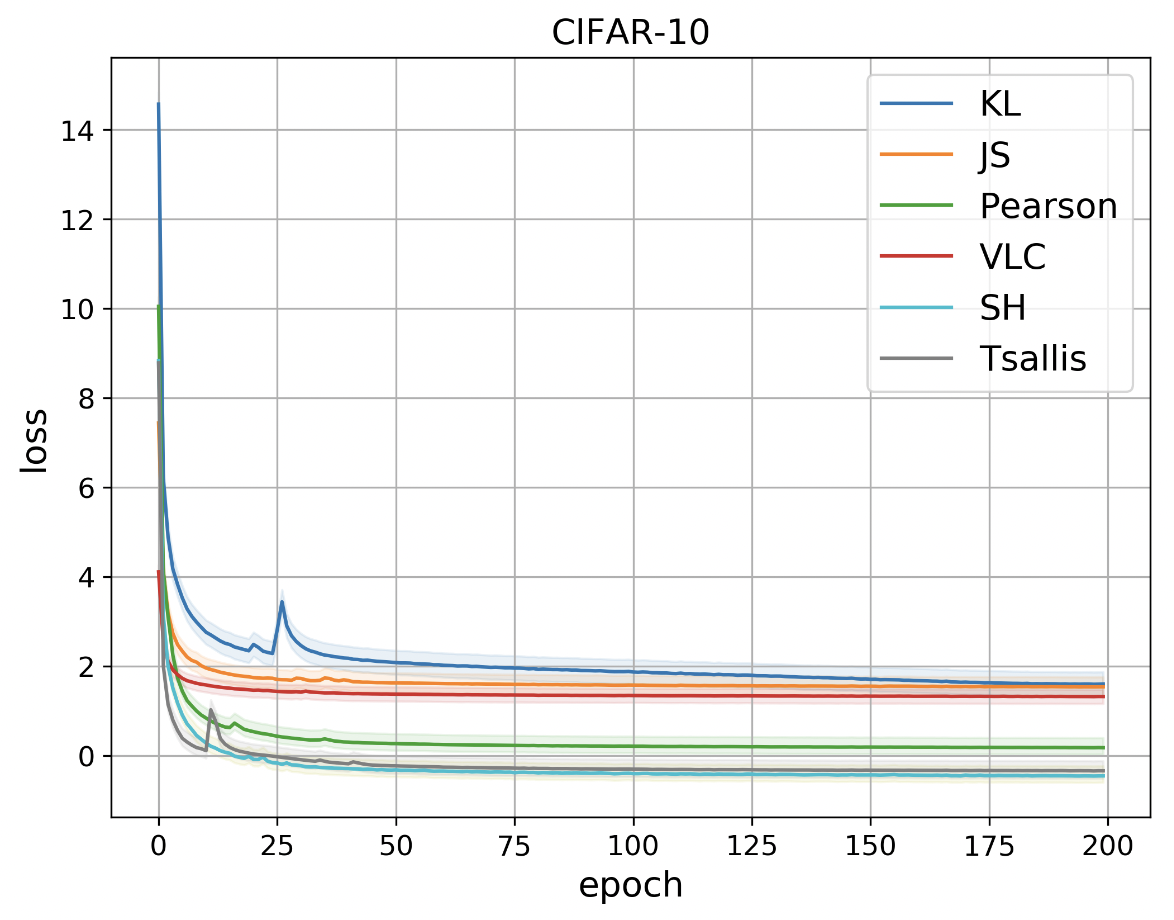}
    \caption{The training loss curves of various $f$-divergences on CIFAR-10 with $200$ epochs.}
    \label{fig:curve}
\end{figure}

\paragraph{$k$-NN evaluation and additional $f$-divergences}
We show more detailed results of Table \ref{tab:accs} in Table \ref{tab:accs_full}, including experiments using $k$-nearest neighbour ($k$-NN) evaluation. Additionally, we have added experiments on other $f$-divergences such as Squared Hellinger and Tsallis-$\alpha$ divergences.

% \paragraph{Variation of performance (\blue{delete})}

% Due to computing resource limit, it is very expensive to repeat all our experiments a few times to see the variation of performance. To illustrate this variation, we show the margins of error of our methods and baselines on CIFAR-10 in Table \ref{tab:error}. We provide the mean and the margin of error with 99.99\% confidence level for each objective. The results show that our method is relatively stable over multiple runs while keeping the overall good performance. 

\begin{figure*}
    \centering\includegraphics[width=0.9\textwidth]{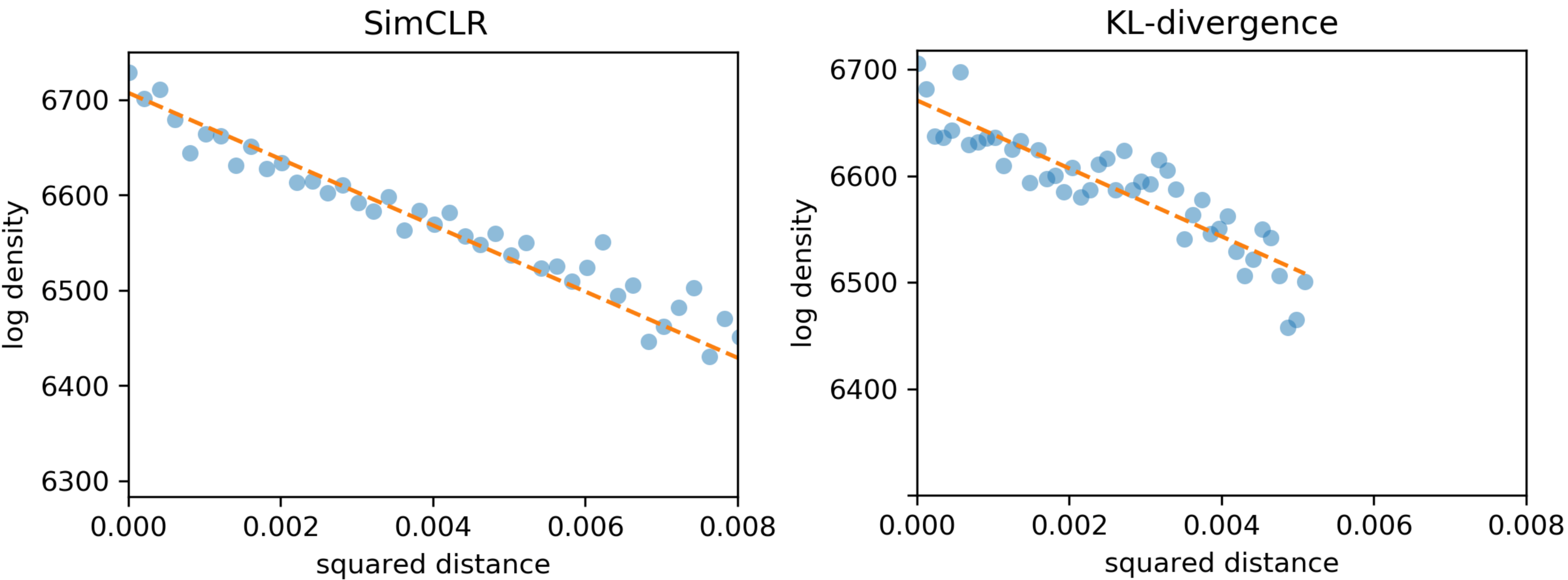}
    \caption{Experiment for verifying Assumption \ref{assmp:joint_vMF}. We draw the relation between the squared distances $\|x^g - y^g\|^2$ and the averaged $\log p_g$ with RealNVP. The features are learned by different algorithms trained on CIFAR-10. ({\bf left}) SimCLR; {\bf (right)} $f$-MICL with the KL divergence.
    }
    \label{fig:gaussian2}
\end{figure*}

\iffalse
\begin{figure*}
    \centering
    \includegraphics[height=5.76cm,width=2cm]{images/gaussian_3.png}
    %\vspace{-0.2em}
    \caption{Additional experiment for verifying Assumption \ref{assmp:joint_vMF}. Here we take 5 different combinations of random data augmentations and draw the relation between the squared distances $\|x^g - y^g\|^2$ and the averaged $\log p_g$ with RealNVP. \blue{({\bf left column}) Gaussian prior; {\bf (right column)} Uniform prior.}
    The features are learned by SimCLR trained on CIFAR-10.
    }
    \label{fig:gaussian3}
\end{figure*}
\fi

\paragraph{Verification of Assumption \ref{assmp:joint_vMF}}
Throughout our paper we made an assumption (Assumption \ref{assmp:joint_vMF}) that the joint feature distribution is a Gaussian kernel. However, is it a valid assumption? In this experiment, we try to show some empirical evidence that this assumption 
approximately holds in practice. Recall that Assumption \ref{assmp:joint_vMF} says that the joint feature distribution of positive pairs is:
\begin{align}\label{eq:gaussian_assumption_app}
p_g(x^g, y^g) \propto \exp\left(-\frac{\|x^g - y^g\|^2}{2\sigma^2}\right)
\end{align}
if the RBF kernel is Gaussian. In order to estimate the joint density of positive pairs, we use normalizing flows, which is a popular method for density estimation. Popular normalizing flow models include NICE \citep{dinh2014nice}, RealNVP \citep{dinh2016density} and Glow \citep{kingma2018glow}. \Cref{eq:gaussian_assumption_app} is equivalent to the following:
\begin{align}
\log p_g(x^g, y^g) = -\frac{\|x^g - y^g\|^2}{2\sigma^2} + {\rm const},
\end{align}
and thus it suffices to show that the log likelihood is linear w.r.t.~the distances between each positive pair. In Figure \ref{fig:gaussian2}, we plot the relation between $\log p_g$, estimated by RealNVP \footnote{Code available at \url{https://github.com/ikostrikov/pytorch-flows}.}  with a Gaussian prior, and the squared distances $\|x^g - y^g\|^2$. The representations are learned by SimCLR, and $f$-MICL with the KL divergence on the CIFAR-10 dataset. To alleviate the estimation error in the flow model, we divide the distances into small intervals and compute the average log-likelihood within each interval. We can see that the log-likelihood is roughly linear w.r.t.~the squared distance, and thus verifying our Assumption \ref{assmp:joint_vMF}.  

%\blue{Moreover, in \Cref{fig:gaussian3}, we also show the estimation by training RealNVP with a uniform prior. Over 5 different combinations of random data augmentations, we observe that the linear relationship generally holds and the estimations by Gaussian prior and uniform prior are very similar. }

\begin{table*}[t]
\centering
\caption{Test accuracy (\%) on the smaller vision datasets. For the Wikipedia dataset we evaluate the semantic textual similarity (STS) via the Spearman's correlation. For each method, we take three separate runs, and show the mean and stand derivation. }
\label{tab:accs_full}
%\small
\setlength\tabcolsep{2pt}
\begin{tabular}{cc@{\hskip2ex}ccc@{\hskip3ex}cccccc}
\toprule
\multirow{2}{*}[-.6ex]{\bf Evaluation} & \multirow{2}{*}[-.6ex]{\bf Dataset} & \multicolumn{3}{c}{\bf Baselines} & \multicolumn{6}{c}{\bf \hspace{1.0em} $f$-MICL}\\
\cmidrule(l{0pt}r{10pt}){3-5}\cmidrule(l{-1pt}r{-1pt}){6-11}
& & MoCo & SimCLR %& MoCo
& Uniformity %& RPC %& RKL & Neynman 
& \phantom{kk}KL\phantom{kk} & JS & Pearson &SH & Tsallis & VLC \\
\midrule
& \multirow{2}{*}{CIFAR-10} & 90.30 & 89.71 %& 88.66 
& 90.41 %& 90.39  % & 10.00 & 10.00 
& \textbf{90.61} & 89.66
& 89.35 & 89.52 & 89.15 & 89.13\\
& & $\pm 0.19 $  & $\pm0.37$ & $\pm0.26$ %& $\pm0.25$ 
& $\mathbf{\pm0.47}$ & $\pm0.28$ & $\pm 0.52$ &$\pm 0.25$ & $\pm0.42$& $\pm0.33$ \\
\cmidrule(l{0pt}r{0pt}){3-11}

\multirow{2}{*}{Linear}

& \multirow{2}{*}{STL-10} & 83.69 & 82.97 & 84.44  %&82.41
&85.33 & {\bf 85.94} & 82.64 & 82.80 & 84.79 & \bf 85.94 \\
& & $\pm 0.22 $  & $\pm 0.32 $ & $\pm 0.19 $ %& $\pm 0.14 $
& $\pm 0.39 $ & $\mathbf{\pm0.17}$ & $\pm 0.37$ &$\pm 0.27$ & $\pm 0.34 $& $\bf \pm0.72$ \\
\cmidrule(l{0pt}r{0pt}){3-11}

& \multirow{2}{*}{TinyImageNet} & 35.72 & 30.56 &41.20 & 34.95 %& 39.46 
& 42.98  &\bf 43.45 & 40.83 & 32.99 & 38.65 \\
& & $\pm 0.17 $  & $\pm 0.28 $ & $\pm 0.19 $ %& $\pm 0.25 $ 
& $\pm 0.20 $ & $\pm0.18$ & \bf $\pm 0.54$ &$\pm 0.67$ & $\pm 0.49 $& $\pm 0.45$ \\

\midrule
& \multirow{2}{*}{CIFAR-10} &  88.70 &  84.92 %& 88.70 
& 89.42 %& 84.21 %& 10.00 & 10.00 
& 89.34 & 89.12 & {\bf 89.44} & 88.13 & 89.18 & 89.15\\

& & $\pm 0.22 $  & $\pm0.39$ & $\pm0.18$ %& $\pm0.24$ 
& $\pm0.57$ & $\pm0.38$ & $\mathbf{\pm 0.60}$ &$\pm 0.18$ & $\pm0.62$& $\pm0.23$ \\
\cmidrule(l{0pt}r{0pt}){3-11}

\multirow{2}{*}{$k$-NN}  
& \multirow{2}{*}{STL-10} & 78.77  & 74.34 & 79.57 %& 73.27
&79.99 & {\bf 80.45} & 76.64 & 78.31 & 76.11 & 79.34 \\
& & $\pm 0.25 $  & $\pm 0.14 $ & $\pm 0.52 $ %& $\pm 0.40 $ 
& $\pm0.47$ & $\mathbf{\pm0.19}$ & $\pm 0.26$ & $\pm 0.33$ & $\pm 0.24$ & $\pm 0.62$ \\
\cmidrule(l{0pt}r{0pt}){3-11}
& \multirow{2}{*}{TinyImageNet} &36.22 & 29.60 & 37.44 %& 24.25 
& 36.17 & \textbf{38.20} & 38.14 & 35.56 & 33.11 & 35.21\\
& & $\pm 0.20 $  & $\pm 0.39 $ & $\pm 0.27 $ %& $\pm 0.35 $ 
& $\pm0.29$ & $\mathbf{\pm0.26}$ & $\pm 0.63$ &$\pm 0.77 $ & $\pm 0.52 $& $\pm 0.33$ \\

\midrule
\multirow{2}{*}{STS}& \multirow{2}{*}{Wikipedia} & 77.88 & 77.40 & 77.95 %& 68.32 
& {\bf 78.02} & 76.76 & 77.59 & 73.60 & 72.68 & 55.07\\

& & ${\pm 0.15}$ & ${\pm 0.12}$ & ${\pm 0.08}$ %&  ${\pm 0.23}$ %& -- & 
& $\mathbf{{\pm 0.13}}$ & ${\pm 0.09}$ & ${\pm 0.12}$ &  ${\pm 0.10}$ & ${\pm 0.09}$ & ${\pm 0.13}$\\

\bottomrule
\end{tabular}
\end{table*}

\end{document}